\pdfoutput=1
\documentclass{article}

\usepackage[preprint]{colm2025_conference}
\usepackage{microtype}
\usepackage{hyperref}
\usepackage{url}
\usepackage{amsmath,amssymb}
\usepackage{amsthm}
\newtheorem{theorem}{Theorem}[section]
\newtheorem{lemma}[theorem]{Lemma}
\usepackage{booktabs}
\usepackage{enumitem}
\usepackage{lineno}
\definecolor{darkblue}{rgb}{0, 0, 0.5}
\hypersetup{colorlinks=true, citecolor=darkblue, linkcolor=darkblue, urlcolor=darkblue}

\usepackage{latexsym}

\usepackage[T1]{fontenc}
\usepackage[utf8]{inputenc}

\usepackage{microtype}

\usepackage{inconsolata}
\usepackage{latexsym}
\usepackage{booktabs} % For improved table line quality
\usepackage{CJKutf8}
\usepackage{tabularx}
\usepackage{amsmath}
\usepackage{subcaption} % 引入subcaption包
\usepackage{multirow}
\usepackage{amsmath}
\usepackage{algorithm}
\usepackage{algpseudocode}
\usepackage{xcolor} % for colors
\usepackage{mdframed} % for framing
\usepackage{lipsum} 
\definecolor{lightgray}{gray}{0.9}
\usepackage{listings}
\usepackage{amssymb}
\usepackage{booktabs}
\usepackage{colortbl}
\usepackage{booktabs}
\usepackage{adjustbox}
\usepackage{multirow}
\usepackage{siunitx}
\usepackage{xcolor}
\usepackage{booktabs}
\usepackage{caption}
\usepackage{adjustbox}
\usepackage{multirow}
\usepackage{siunitx}
\usepackage{caption}
\usepackage{float}   
\usepackage{makecell}
\usepackage{amssymb} % for \downarrow
\usepackage[utf8]{inputenc}
\usepackage{caption}  % for adjusting caption 
\usepackage{float}
\usepackage{stfloats}
\usepackage{microtype}
\usepackage{hyperref}
\usepackage{tcolorbox}

\usepackage{tikz}
\usepackage{graphicx}
\definecolor{headergray}{gray}{0.85}
\definecolor{rowgray}{gray}{0.95}
\definecolor{bestgreen}{rgb}{0.0, 0.6, 0.0}
\newtcolorbox{mybox}{
  colback=red!10, % Background color is light red (10% red)
  colframe=red!75!black, % Frame color is dark red (75% red + 25% black)
  boxrule=1pt, % Frame width is 1pt
  arc=4pt, % Corner radius
  boxsep=0pt, % Space between text and frame
  left=6pt, % Space between text and left border
  right=6pt, % Space between text and right border
  top=6pt, % Space between text and top border
  bottom=6pt, % Space between text and bottom border
  width=\textwidth % Box width set to the text width
}

\newtcolorbox{mybox1}{
  colback=gray!10, % Background color is light gray (10% gray)
  colframe=gray!75!black, % Frame color is dark gray (75% gray + 25% black)
  boxrule=1pt, % Frame width is 1pt
  arc=4pt, % Corner radius
  boxsep=0pt, % Space between text and frame
  left=6pt, % Space between text and left border
  right=6pt, % Space between text and right border
  top=6pt, % Space between text and top border
  bottom=6pt, % Space between text and bottom border
  width=\textwidth % Box width set to the text width
}

\title{SAGE: Self-evolving Agents with Reflective and Memory-augmented Abilities}

\author{
  \begin{tabular}{c}
    Xuechen Liang $^{1}$\thanks{Equal contribution.}, 
    Yangfan He $^{2*}$, 
    Yinghui Xia $^{3*}$, 
    Xinyuan Song $^{4}$, \\
    Jianhui Wang $^{5}$, 
    Meiling Tao $^{6}$, 
    Li Sun $^{7}$,
    Xinhang Yuan $^{8}$,
    Jiayi Su $^{9}$, \\
    Keqin Li $^{10}$,
    Jiaqi Chen $^{10}$,
    Jinsong Yang $^{3*}$, 
    Siyuan Chen $^{10}$,
    Tianyu Shi $^{11}$\thanks{Corresponding author: \texttt{ty.shi@mail.utoronto.ca}} \\[1ex]
    $^1$ East China Jiao Tong University,
    $^2$ University of Minnesota - Twin Cities, \\
    $^3$ AutoAgents.ai,
    $^4$ Emory University, \\
    $^5$ University of Electronic Science and Technology of China, \\
    $^6$ Guangdong University of Technology, \\
    $^7$ Amazon,
    $^8$ Washington University,
    $^9$ Xiamen University,\\
    $^{10}$ Independent Researcher,
    $^{11}$ University of Bristol,
    $^{12}$ University of Toronto \\ 
    { \textbf{Project Page: \url{https://jianhuiwemi.github.io/SAGE}}} 
  \end{tabular}
}

\begin{document}
\maketitle
\begin{abstract}
Large language models (LLMs) have made significant advances in the field of natural language processing, but they still face challenges such as continuous decision-making, lack of long-term memory, and limited context windows in dynamic environments. To address these issues, this paper proposes an innovative framework—Self-evolving Agents with Reflective and Memory-augmented Abilities (\textbf{SAGE}). The SAGE framework comprises three agents: the User, the Assistant, and the Checker. By integrating iterative feedback, reflective mechanisms, and a memory optimization mechanism based on the Ebbinghaus forgetting curve, it significantly enhances the agents' capabilities in handling multi-tasking and long-span information. The agents, through self-evolution, can adaptively adjust strategies, optimize information storage and transmission, and effectively reduce cognitive load. We evaluate the performance of the SAGE framework on multiple benchmarks and long text tasks. Experimental results show that SAGE significantly improves model performance, achieving a 2.26X improvement on closed-source models and an improvement ranging from 57.7\% to 100\% on open-source models, with particularly notable effects on smaller models. 
\end{abstract}
\vspace{-0.5em}
\section{Introduction}
\vspace{-0.5em}
In recent years, large language models (LLMs) have made significant progress in the field of natural language processing, demonstrating powerful performance in tasks such as dialogue and text generation~\cite{brown2020language, he2025enhancing, he2024enhancing}. Recently, there has been growing interest in applying LLMs as autonomous agents (LLM agents), which use language not only for understanding and generation, but also for planning and acting in interactive environments~\cite{yao2022react, shinn2023reflexion, liang2024cmat, li2024falcon, zhou2024human}. However, these models still face several challenges: (1) LLM Agents need to continuously make decisions in changing environments and adapt to new situations and tasks. (2) LLM Agents lack long-term memory mechanisms, which is increasingly evident in situations requiring sustained interaction with the environment~\cite{graves2016hybrid}. The limited context window also hinders the model's ability to handle information over long time spans~\cite{rae2019compressive}.

To tackle these challenges, researchers have proposed meta-learning and multi-task learning to enhance the transferability and adaptability of LLM agents. For memory limitations, prior works like MemGPT~\cite{packer2024memgpt} use a FIFO queue to manage forgetting, while MemoryBank employs a forgetting curve based on insertion time. However, these approaches are often task-specific, lacking a general framework to systematically improve LLM agents in complex environments. Recent innovations, such as AutoGPT~\cite{yang2023autogptonlinedecisionmaking} and BabyAGI~\cite{nakajima2024babyagi}, leverage LLMs as core controllers, aiming to solve real-world challenges. Yet, multi-agent frameworks still face issues like communication overload, heavily relying on memory to maintain context. As interaction history grows, resource demands and latency increase, limiting efficient deployment in practical scenarios.

In this paper, we propose an innovative framework, Self-evolving Agents with reflective and memory-augmented abilities (SAGE). By enhancing agents' self-adjustment capabilities through reflection, they can more effectively utilize historical information and make efficient decisions when faced with complex and dynamic tasks. From the perspective of self-evolution, we introduce a memory optimization mechanism based on the Ebbinghaus forgetting curv~\cite{ebbinghaus1885gedachtnis}. This mechanism helps agents selectively retain key information, optimize information storage and transmission, reduce unnecessary cognitive load, and enhance agents' capabilities in interaction tasks with the environment. 
%Experimental results show that our approach achieves significant improvements across various benchmarks, particularly excelling in smaller models. Specifically, on AgentBench~\cite{liu2023agentbench}, the performance of powerful LLMs like GPT-3.5 and GPT-4 is enhanced by up to 2.26X. For open-source models, performance improvements range from 57.7\% to 100\%. In tasks such as multi-source question answering and code generation, our approach achieves state-of-the-art results\cite{etezadi2023state}.
Experimental results demonstrate that our approach consistently enhances the performance of both proprietary and open-source LLMs across a wide range of benchmarks. The improvements are especially notable in smaller models, where the gains are more pronounced. On tasks such as multi-source question answering and code generation, our method sets a new standard, outperforming existing techniques and achieving leading benchmarks~\cite{etezadi2023state}, including AgentBench~\cite{liu2023agentbench}.

The main contributions of our work are as follows:
\begin{itemize}[leftmargin=10pt]
    \item We propose a novel framework, SAGE, which introduces a reflection mechanism to enhance agents' self-adjustment capabilities. Without any additional training, this enables agents to utilize historical information more effectively and make better decisions when faced with complex and dynamic tasks.
    \item We introduce a memory optimization mechanism based on the Ebbinghaus forgetting curve. This helps agents selectively retain key information, reducing the issue of information overload in multi-agent systems.
    \item SAGE achieves improvements over strong baselines in multiple challenging real-world tasks and attains state-of-the-art results on benchmarks. This framework can be applied to other LLMs, with particularly strong improvements in smaller models.
\end{itemize}
\vspace{-0.5em}
\section{Related work}
\vspace{-0.5em}
\subsection{Self-Improvement of Reasoning and Decision-Making}
\vspace{-0.5em}
Deep learning has transformed multiple domains including NLP, time series analysis and computer vision~\cite{qiu2025easytime, qiu2025duet, qiu2024tfb}. A lot of research is focused on making large language models (LLMs) better at improving themselves. Some researchers are working on using carefully crafted prompts to help models learn how to get better, although this usually only works for one-off tasks. Others are tweaking how models get feedback during tasks, which helps them get better at thinking things through~\cite{huang2022large}. There's also work on using strategies like random beam searches to help models make smarter decisions and assess their own work. Most current methods rely on quick, one-off tweaks and learning strategies that need lots of resources and hands-on tech help~\cite{tian2024selfimprovement}. This paper introduces a self-reflection mechanism, showing that LLMs can keep getting better and produce higher quality work across different tasks, all without needing extra training.

% Many studies focus on enhancing the self-improvement capabilities of LLMs. Pryzant et al. optimize self-improvement through semantic prompt writing, but this method is limited to single-generation tasks. Paul et al. improve reasoning responses by fine-tuning a critic model to provide intermediate feedback within the trajectory. Xie et al. use stochastic beam search actions to execute more efficient decision search strategies, enabling agents to utilize a self-assessment component. In recent years, the self-improvement capabilities of LLMs have revealed the future direction of LLM self-evolution. Although current LLM development methods rely on single-round instructional fine-tuning and reinforcement learning-based approaches, these strategies require substantial resources and continuous human intervention. This paper introduces a self-reflection mechanism, demonstrating that LLMs can provide iterative self-refinement without additional training, thereby generating higher-quality outputs across a wide range of tasks.
\vspace{-0.5em}
\subsection{Memory Mechanism for LLM-based Agents}
\vspace{-0.5em}
In LLM-based agents, the memory module stores, processes, and retrieves task-related information, supporting knowledge accumulation, experience handling, and decision-making. To enhance the self-evolution capabilities of these agents, researchers are focused on designing and optimizing these memory modules~\cite{raffel2020exploring}.
Past research has covered various designs and implementations of memory modules. This includes integrating information from different trials to boost reasoning abilities or storing information in natural language to enhance the module's interpretability and user-friendliness~\cite{wada2019unsupervised}. 
%Additionally, some studies concentrate on developing memory read-write operations, helping agents effectively interact with their environment and complete tasks.
Despite progress, self-adjustment and memory management still need improvement to handle complex real-world problems more effectively.
%Our proposed memory optimization mechanism is designed to help agents better manage and adapt to dynamic and complex task environments.

% In LLM-based agents, the memory module is considered one of the critical components for storing, processing, and retrieving information relevant to the agent's tasks. Memory plays a crucial role in determining how the agent accumulates knowledge, processes historical experiences, and supports its actions. Therefore, to enhance the self-evolution capabilities of LLM-based agents, researchers are focused on designing and optimizing memory modules. 

% Past research has explored various designs and implementations of memory modules. For example, some researchers combine information from trials and cross-trials to construct memory modules, thereby enhancing the agent's reasoning abilities. Other researchers store memory information in natural language form to improve the module's interpretability and user-friendliness. Additionally, some studies focus on designing memory read-write operations, enabling agents to interact effectively with their environment and complete tasks. 

% Although past research has made some progress in the design and implementation of memory modules, further improvement in the self-adjustment capabilities and memory management efficiency of LLM-based agents is still needed to address complex problems in real-world applications. Therefore, our approach introduces a memory optimization mechanism, allowing agents to better cope with complex and dynamic task environments.
\begin{figure*}[!ht]
    \centering
    \includegraphics[width=\linewidth]{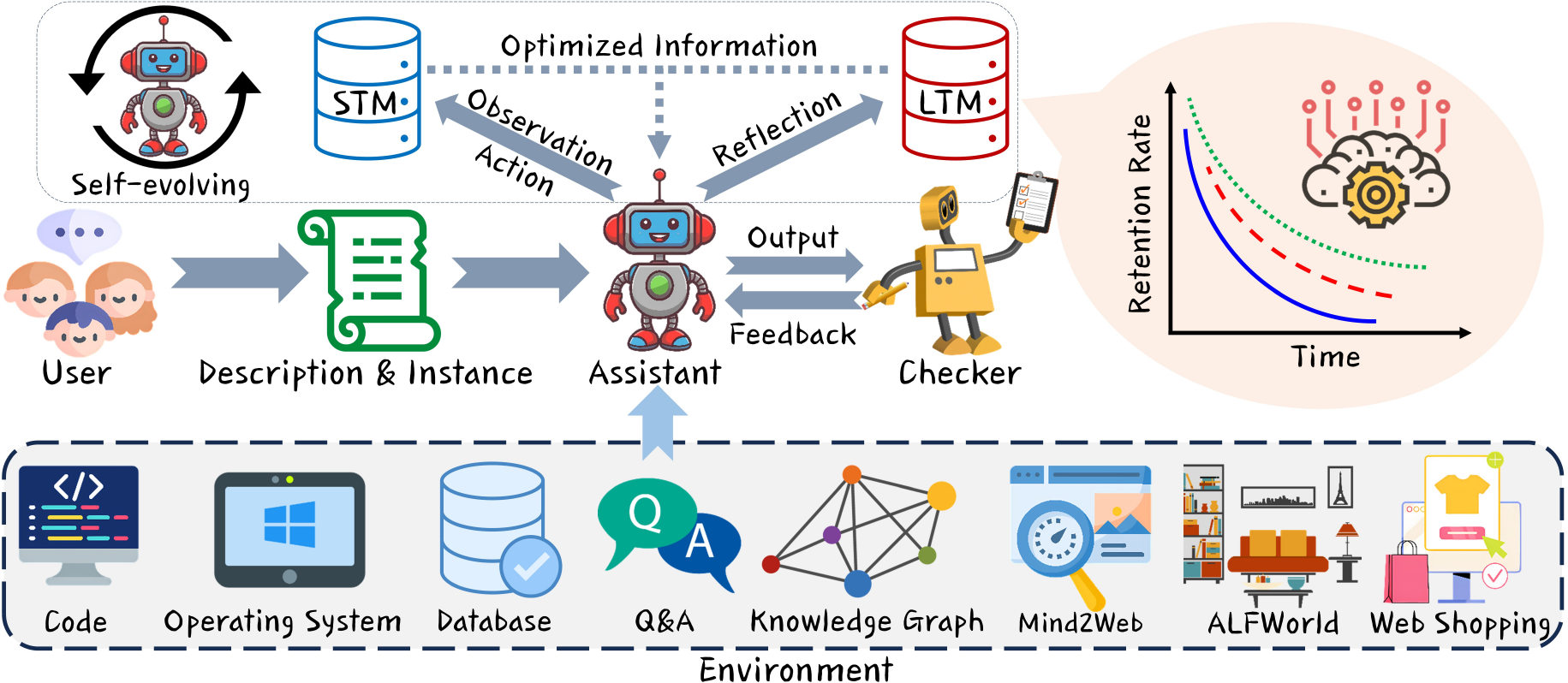}
    \caption{An illustration of the SAGE: a user provides a description and instance to the assistant with short-term (STM) and long-term (LTM) memory. The assistant performs observation, action, reflection, and output, which the checker reviews. The retention rate curve on the right illustrates memory decay over time, with a self-evolving loop guiding continued updates.}
    \label{fig:sage_framework}
\end{figure*}
\vspace{-0.5em}
\section{Method}
\vspace{-0.5em}
In this section, we present the SAGE framework, designed to improve agent performance by leveraging three core mechanisms: iterative feedback, reflection, and MemorySyntax (as shown in Figure~\ref{fig:sage_framework}). The assistant agent $A$ iteratively updates its policy $\pi_{\theta}$ based on feedback $f_t$ provided by the checker agent $C$, optimizing over successive iterations to maximize the expected reward $R$. The reflection mechanism allows $A$ to incorporate historical observations $\mathcal{O}_t$ and actions $\mathbf{a}_t$, forming a self-reflection $r_t$, which is stored in the memory $\mathcal{M}_L$ for future decision-making. Finally, MemorySyntax combines the Ebbinghaus forgetting curve with linguistic principles to manage memory decay, dynamically updating the agent's short-term memory $\mathcal{M}_S$ and long-term memory $\mathcal{M}_L$ by prioritizing information based on its retention strength $S(I_t)$, thus improving the agent's ability to retain crucial information while discarding less relevant data. The subsequent subsections detail these components.

\vspace{-0.5em}
\subsection{Iterative Feedback}
\vspace{-0.5em}
The iterative feedback mechanism in the SAGE framework enables the assistant agent $A$ to refine its policy $\pi_{\theta}$ through repeated interactions with the checker agent $C$. At each iteration $t$, the assistant receives feedback $f_t$ based on its current output $\mathbf{o}_t$, and adjusts its policy accordingly. This process continues until the checker validates the output or the iteration cap $N$ is reached, ensuring that $A$ incrementally optimizes its decisions to improve task performance over successive iterations.
\vspace{-0.5em}
\subsubsection{Initialization Phase}
\vspace{-0.5em}
\textbf{Role Assignment.}  
In the SAGE framework, three agents are introduced: the user $U$, the assistant $A$, and the checker $C$. The user, upon receiving prompt $P_U$, assumes the role of task proposer by specifying a task $\mathcal{T}_U$ and related constraints $\mathcal{C}_U$. The assistant, upon receiving prompt $P_A$, generates a sequence of actions $\mathbf{a}_t$ based on the observations $\mathcal{O}_t$ and environment $\mathcal{E}$. The checker $C$ evaluates the output $\mathbf{o}_A$ produced by the assistant, providing feedback $f_C$ based on the discrepancy between $\mathbf{o}_A$ and the expected result, updating its policy $\pi_{\theta}$ iteratively to minimize this gap.\\
\textbf{Task Assignment.}  
The task $\mathcal{T}_U$ provided by the user includes an initial task description $\mathbf{d}_U$ and an instance $\mathbf{i}_U$ that serves as the reference for correct output. This forms the input set $\mathcal{I}_A = (\mathbf{d}_U, \mathbf{i}_U)$ for the assistant to initiate its generative process. The assistant then proceeds by selecting an action $\mathbf{a}_t$ at each time step $t$, guided by $\pi_{\theta}$, with the goal of maximizing the reward $R_t$ for completing $\mathcal{T}_U$.
\begin{figure*}[!ht]
    \centering
    \includegraphics[width=\linewidth]{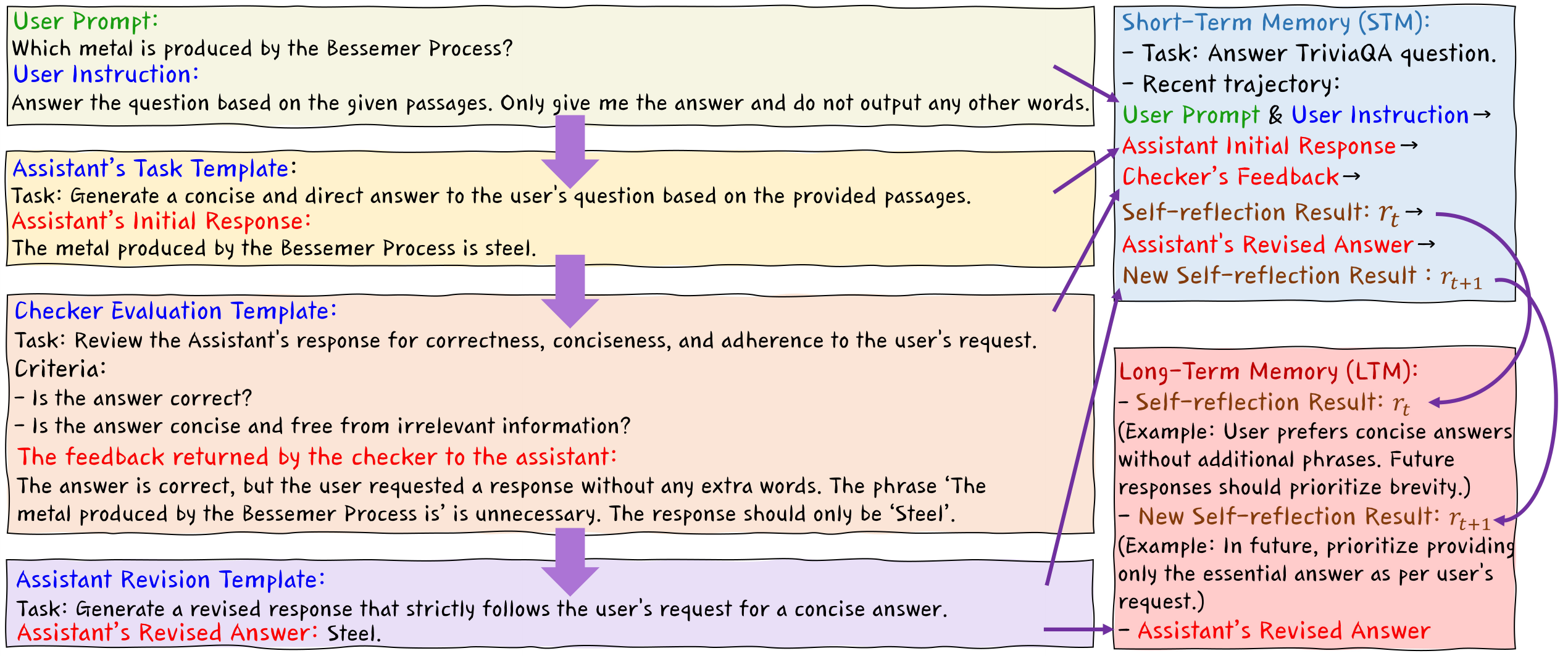}
    \caption{An example of the assistant's iterative workflow, including checker evaluation, prompt templates for feedback, and reflection processes integrating short-term and long-term memory.}
    \label{fig:assistant_refinement}
\end{figure*}

\vspace{-0.5em}
\subsubsection{Actual Interaction Phase}
\vspace{-0.5em}
Following the role assignment and task definition in the initialization phase, the assistant $A$ transitions into the actual interaction phase to generate outputs aimed at accomplishing the task $\mathcal{T}_U$. In this phase, $A$ iteratively produces outputs $\mathbf{o}_t$ at each time step $t$ based on the task description $\mathbf{d}_U$ and instance $\mathbf{i}_U$ provided in the input set $\mathcal{I}_A = (\mathbf{d}_U, \mathbf{i}_U)$. At each time step \( t \), the assistant selects an action \( \mathbf{a}_t \) by following its policy \( \pi_{\theta} \), which is conditioned on the current state \( s_t \), the reward signal \( R_t \) (the reward score for task performance), and feedback \( f_t^i \) from the checker \( C \). This decision-making process is formalized as:
\vspace{-0.3em}
\begin{equation}
    \mathbf{o}_t \sim \pi_{\theta}(\mathbf{o}_t \mid s_t, R_t, f_t^i),
\end{equation}
where \( \pi_{\theta} \) represents the assistant's policy, \( R_t \) reflects the reward signal based on task performance at time \( t \), and \( f_t^i \) is the feedback provided by the checker during the \( i \)-th iteration.

As the interaction progresses, the checker $C$ evaluates each output $\mathbf{o}_t$ generated by $A$, comparing it against the expected outcome derived from $\mathbf{i}_U$. Based on this comparison, the checker provides iterative feedback $f_t^i$ to guide $A$ in refining its actions $\mathbf{a}_t$ and outputs $\mathbf{o}_t$. The iterative refinement continues until either the checker validates the output as correct or the iteration limit $N$ is reached.

\noindent
\textbf{Theoretical optimality of iterative feedback mechanism.}
% The user ($U$), assistant ($A$), and checker ($C$) can be treated as three players in a game-theoretic model, each with its own objectives and strategies. The user provides a task description and constraints ($\mathcal{T}_U, \mathcal{C}_U$); the assistant follows a policy $\pi_{\theta}$ to generate outputs $\mathbf{o}_t$ with the aim of maximizing a task completion reward $R$; and the checker supplies feedback $f_t$ to reduce the error between $\mathbf{o}_t$ and the expected result. 
In the SAGE framework, the assistant repeatedly updates its policy through this checker feedback, enabling the outputs to be incrementally refined until either the result is validated or a specified iteration limit is reached. The assistant’s utility $R_A$ reflects task performance, and the checker’s utility $R_C$ depends on its feedback. The following theorem indicates that this iterative feedback mechanism leads to strategy stability in the sense of a Nash equilibrium~\cite{fudenberg1991game}.

\begin{theorem}[Theory for the multi-agent iterative feedback system]
\label{thm:nash_stability}
Let $\mathcal{U}, \mathcal{A}, \mathcal{C}$ denote the compact, convex strategy spaces of the user (U), assistant (A), and checker (C), respectively. Assume that the utility functions
\begin{equation}
R_U: \mathcal{U} \times \mathcal{A} \times \mathcal{C} \to \mathbb{R},\quad R_A: \mathcal{U} \times \mathcal{A} \times \mathcal{C} \to \mathbb{R},\quad \text{and}\quad R_C: \mathcal{U} \times \mathcal{A} \times \mathcal{C}\to \mathbb{R}
\end{equation}
are continuous in each player’s strategy. Then, by the Debreu-Glicksberg-Fan fixed-point theorem, there exists a Nash equilibrium 
\begin{equation}
(s_U^*, s_A^*, s_C^*) \in \mathcal{U} \times \mathcal{A} \times \mathcal{C}.
\end{equation}
Furthermore, suppose that the assistant’s policy $\pi_\theta$ is updated via policy gradient methods and that the checker’s strategy is refined through convex optimization. Then, the iterative update procedures yield sequences 
\begin{equation}
\{\pi_{\theta}^{(k)}\}_{k\geq 0} \quad \text{and} \quad \{f^{(k)}\}_{k\geq 0},
\end{equation}
which converge to a stable strategy profile $(\pi_{\theta}^*, f^*)$, and has:
\begin{equation}
R_A(\pi_{\theta}^*, f^*) \geq R_A(\pi_{\theta}, f^*), \quad
R_C(\pi_{\theta}^*, f^*) \geq R_C(\pi_{\theta}^*, f).
\end{equation}
\end{theorem}
This result demonstrates that the iterative feedback mechanism enhances the model’s strategy stability by converging to a Nash equilibrium in the three-player game. It provides a stronger justification for the three-agent system versus simpler alternatives (such as \textbf{two-agent systems}). For the detailed theoretical explanation and proof, see~\ref{sec:thm2}.

\vspace{-0.5em}
\subsubsection{Evolutionary Goals and Directions}
\vspace{-0.5em}
Leveraging the feedback $f_t^i$ obtained at each iteration $t$, the assistant $A$ formulates new evolutionary objectives:
\begin{equation}
    \mathcal{G}^{t+1} = (\mathcal{A}^{t+1}, \mathcal{D}^{t+1}), \quad 
\mathcal{D}^{t+1} = \arg\min_{\mathcal{D}_t \in \Delta} \sum_{i\in I_t} L_D\left(\mathcal{D}_t; f_t^i, \pi_{\theta}^{t}\right),
\end{equation}
where $\mathcal{A}^{t+1}$ represents the updated memory optimization mechanisms, and $\mathcal{D}^{t+1} \in \Delta$ refers to the model's self-adjustments to make the RL algorithm converge. These evolutionary objectives guide the assistant in updating its policy $\pi_{\theta}$ for the subsequent iteration. The policy update is governed by the function $\psi$, which integrates the current policy $\pi_{\theta}^t$ with the new evolutionary objectives $\mathcal{G}^{t+1}$:
\vspace{-0.1em}
\begin{equation}
\theta^{t+1} = \phi\bigl(\theta^t, \mathcal{G}^{t+1}\bigr)
= \theta^t + \alpha  \nabla_\theta \Bigl[\lambda_A L_A\bigl(\theta^t,\mathcal{A}^{t+1}\bigr) + \lambda_D L_D\bigl(\theta^t,\mathcal{D}^{t+1}\bigr)\Bigr].
\end{equation}
\vspace{-0.1em}
Here $L_A(\theta, \mathcal{A})$ and $L_D(\theta, \mathcal{D})$ are MSE loss functions corresponding to the memory-optimization and self-adjustment aspects, respectively, and $\lambda_A, \lambda_D \ge 0$ are weighting coefficients. The iterative policy refinement enables the assistant $A$ to continuously adapt its strategies based on cumulative feedback and evolving task requirements, thereby improving its overall performance in dynamic environments.

\vspace{-0.5em}
\subsection{Memory Management}
\vspace{-0.5em}
The SAGE framework implements a dual-memory system, consisting of Short-Term Memory (STM) and Long-Term Memory (LTM), to manage task-relevant information and enhance the agent's reasoning and decision-making capabilities (see Figure~\ref{fig:assistant_refinement} for a visual representation of this process).

\noindent \textbf{Short-Term Memory (STM).}  
STM is responsible for storing immediate, task-specific data with limited capacity. It updates rapidly with new observations ($\mathcal{O}_t$) and actions ($\mathbf{a}_t$), maintaining a recent trajectory history $\mathcal{T}_t = (\mathcal{O}_t, \mathbf{a}_t)$. This allows the agent to make real-time decisions and respond quickly to dynamic changes in the environment~\cite{mnih2015human}.

\noindent \textbf{Long-Term Memory (LTM).}  
LTM retains critical information and self-reflections ($r_t$) over extended periods, enabling the agent to accumulate knowledge from past interactions and apply it to future tasks. Stored as $\mathcal{M}_L = \{r_t \mid t \in T\}$, this memory mechanism allows the agent to use prior experiences to improve task performance, particularly in complex environments that require long-span information~\cite{graves2016hybrid}.

By integrating STM and LTM, the SAGE framework allows the agent to balance immediate task demands with the ability to draw from accumulated knowledge, thereby enhancing its overall decision-making efficiency.

% In this way, lessons learned from past experiences are preserved and utilized in future interactions, enhancing the agent's ability to make informed decisions over time.
\vspace{-0.5em}
\subsubsection{Reflection}
\vspace{-0.5em}
Figure~\ref{fig:2} illustrates an example of the reflection mechanism applied to a HotpotQA task~\cite{yang2018hotpotqadatasetdiverseexplainable}. The reflection mechanism equips the assistant $A$ with sparse reward signals, such as binary success/failure states, trajectory $\mathcal{T}_t$, and its stored memory $\mathcal{M}_L$. The assistant processes these inputs, deriving insights from past performance and storing self-reflections $\mathbf{r}_t$ for future decision-making. These self-reflections, richer than scalar rewards, enhance the assistant's learning capacity and are incorporated into long-term memory:
\vspace{-0.1em}
\begin{equation}
    \mathbf{r}_t = \text{ref}(\mathbf{o}_{1:t}, \mathbf{R}_{1:t}),
\end{equation}
where $\text{ref}(\cdot)$ denotes the reflection function based on the output sequence $\mathbf{o}_{1:t}$ and rewards $\mathbf{R}_{1:t}$. The derived reflection $\mathbf{r}_t$ is then added to $\mathcal{M}_L$:
\vspace{-0.1em}
\begin{equation}
    \mathcal{M}_L \leftarrow \mathcal{M}_L \cup \{\mathbf{r}_t\}.
\end{equation}
The process gradually enhances the agent's decision-making, allowing it to adapt effectively through accumulated experience.

\vspace{-0.5em}
\subsubsection{MemorySyntax}
\vspace{-0.5em}
Building upon the reflection mechanism, the MemorySyntax method integrates the Ebbinghaus forgetting curve with linguistic principles to emulate human-like memory processes within the agent's memory management system. Let $I_t$ denote the information received at time $t$, and let $R(I_t, \tau)$ represent its retention rate after a time interval $\tau$. According to the Ebbinghaus forgetting curve, the retention rate is modeled as:
\vspace{-0.1em}
\begin{equation}
    R(I_t, \tau) = e^{-\frac{\tau}{S}},
    \label{eq:forgetting_curve}
\end{equation}
\vspace{-0.2em}
where $S$ is the strength of the information, reflecting its importance and complexity.

To enhance retention, MemorySyntax applies linguistic optimization to $I_t$, producing an optimized version $I_t^*$ with increased strength $S^* > S$. The retention rate for $I_t^*$ is defined as:
\vspace{-0.1em}
\begin{equation}
    R(I_t^*, \tau) = 
    \begin{cases}
        e^{-\frac{\tau}{S^*}}, & \text{if } I_t^* \in \mathcal{M}_S, \\
        e^{-\frac{\tau}{S}},   & \text{if } I_t^* \in \mathcal{M}_L,
    \end{cases}
    \label{eq:optimized_retention}
\end{equation}
\vspace{-0.2em}
where $\mathcal{M}_S$ and $\mathcal{M}_L$ represent short-term memory and long-term memory, respectively.

The agent updates its memory state $\mathcal{M}_t$ based on the retention rate of $I_t^*$ using predefined thresholds $\theta_1$ and $\theta_2$, with $\theta_1 > \theta_2$. The memory update rule is formalized as:
\vspace{-0.1em}
\begin{equation}
    \mathcal{M}_{t+1} =
    \begin{cases}
        \mathcal{M}_t \cup \{ I_t^* \}, & \text{if } R(I_t^*, \tau) \geq \theta_1, \\
        \mathcal{M}_t \setminus \{ I_t^* \}, & \text{if } R(I_t^*, \tau) < \theta_2, \\
        \mathcal{M}_t, & \text{otherwise}.
    \end{cases}
    \label{eq:memory_update}
\end{equation}
\vspace{-0.2em}
This update rule operates under the following conditions:

\begin{itemize}[leftmargin=10pt]
    \item \textbf{Retention in Short-Term Memory ($\mathcal{M}_S$):} If $R(I_t^*, \tau) \geq \theta_1$, the information $I_t^*$ is retained in short-term memory $\mathcal{M}_S$.
    
    \item \textbf{Discarding Information:} If $R(I_t^*, \tau) < \theta_2$, the information $I_t^*$ is considered insignificant and is discarded.
    
    \item \textbf{Transfer to Long-Term Memory ($\mathcal{M}_L$):} If $\theta_2 \leq R(I_t^*, \tau) < \theta_1$, the information $I_t^*$ is transferred to long-term memory $\mathcal{M}_L$.
\end{itemize}
The above MemorySyntax method is optimal and ensures maximal information utility. For theoretical background and detailed proof, see Appendix~\ref{sec:thm1}.

By simulating human memory dynamics, MemorySyntax enables the agent to prioritize essential information, retain critical data in short-term memory $\mathcal{M}_S$, store important but less frequently used information in long-term memory $\mathcal{M}_L$, and discard irrelevant data. The mechanism addresses memory capacity limitations and enhances the agent's ability to perform complex tasks requiring efficient memory management.

\vspace{-0.5em}
\section{Experiment}
\vspace{-0.5em}
To demonstrate the capabilities and performance of SAGE in coordinating autonomous agent groups to work together on tasks, we conduct extensive quantitative experiments on benchmark tasks. We use a public benchmark, AgentBench~\cite{liu2023agentbench}, which is a multidimensional evolutionary benchmark from which we select six tasks. These tasks test the reasoning and decision-making abilities of LLMs acting as agents in multi-turn open-ended generation settings. To further assess the agents' long-context understanding, we select four widely adopted tasks related to long-text processing. These tasks reflect the agents' programming abilities(LCC\cite{guo2023longcoder}, RepoBench-P~\cite{liu2023repobench}) and reasoning abilities(HotpotQA~\cite{yang2018hotpotqadatasetdiverseexplainable}, TriviaQA~\cite{joshi2017triviaqa}).

\vspace{-0.5em}
\subsection{Evaluation on AgentBench}
\vspace{-0.5em}
\noindent\textbf{Task Description}  
AgentBench includes scenarios from CODE (Knowledge Graph, OS, DB), GAME (ALFWorld)~\cite{ALFWorld20}, and WEB (WebShop~\cite{yao2023webshop}, Mind2Web~\cite{deng2023mind2web}). For more details for the datasets and benchmarks, see Appendix \ref{sec:appendixA}.

\noindent\textbf{Baselines}  
We evaluate commercial models GPT-3.5~\cite{brown2020language} and GPT-4~\cite{openai2024gpt4technicalreport}, and open-source models Llama2~\cite{touvron2023llama}, Codellama~\cite{2024codellamaopenfoundation}, Qwen~\cite{bai2023qwen}, and ChatGLM2~\cite{glm2024chatglm}. Dialogue history exceeding the model length limit is truncated, using greedy decoding.

\noindent\textbf{Results}  
As shown in Table~\ref{tab:merged_performance_corrected}, our method significantly improves model performance on AgentBench, especially for smaller models. GPT-3.5 and GPT-4, despite already high scores, also show notable improvements with SAGE, up to 2.26x in the Database task. Llama2-7b is notably enhanced, showing the method's impact on weaker models. CodeLlama-7b and Qwen-1.8B also see substantial gains. Qwen-1.8B, after applying our method, performs close to GPT-3.5, highlighting its potential as a general agent. Llama2, previously error-prone, shows a significant reduction in basic errors through feedback and memory optimization, proving that our method not only enhances agent capabilities but also reduces fundamental errors in complex tasks.

\begin{table*}[!ht]
  \centering
  % \makebox[\textwidth][c]{%
  % \resizebox{1.2\textwidth}{!}{%
  \caption{Baseline and SAGE Framework Performance on AgentBench}
  \label{tab:merged_performance_corrected}
  \scriptsize 
  \renewcommand{\arraystretch}{1.2} 
  \hspace*{-0.6cm}  % 调整这个值控制左右偏移
  \resizebox{1.1\textwidth}{!}{  % 1.25 表示宽度略大于页面宽度，保持左右超出一致
\begin{tabular}{cccccccccccccccc}
    \toprule
    \multirow{2}{*}{\textbf{LLM Type}} & \multirow{2}{*}{\textbf{Model}} & 
    \multicolumn{2}{c}{\textbf{VER}} & 
    \multicolumn{2}{c}{\textbf{OS}} & 
    \multicolumn{2}{c}{\textbf{DB}} & 
    \multicolumn{2}{c}{\textbf{KG}} & 
    \multicolumn{2}{c}{\textbf{ALF}} & 
    \multicolumn{2}{c}{\textbf{WS}} & 
    \multicolumn{2}{c}{\textbf{M2W}} \\
    \cmidrule(lr){3-4} \cmidrule(lr){5-6} \cmidrule(lr){7-8} 
    \cmidrule(lr){9-10} \cmidrule(lr){11-12} \cmidrule(lr){13-14} \cmidrule(lr){15-16}
    & & \textbf{Base} & \textbf{  SAGE} & 
    \textbf{Base} & \textbf{SAGE} & 
    \textbf{Base} & \textbf{SAGE} & 
    \textbf{Base} & \textbf{SAGE} & 
    \textbf{Base} & \textbf{SAGE} & 
    \textbf{Base} & \textbf{SAGE} & 
    \textbf{Base} & \textbf{SAGE} \\
    \midrule
    \multirow{2}{*}{API} 
    & GPT-4           & 42.4 & \textcolor{bestgreen}{\textbf{\ \ 49.7} }& 32.0 & \textcolor{bestgreen}{\textbf{39.8}} & 57.4 & \textcolor{bestgreen}{\textbf{63.1}} & 78.0 & \textcolor{bestgreen}{\textbf{82.0}} & 67.1 & \textcolor{bestgreen}{\textbf{67.8}} & 27.0 & 32.0 & 27.0 & \textcolor{bestgreen}{\textbf{32.0}} \\ 
    & GPT-3.5         & 31.6 & \textcolor{bestgreen}{\textbf{\ 38.3}} & 15.7 & \textcolor{bestgreen}{\textbf{35.6}} & 25.9 & \textcolor{bestgreen}{\textbf{37.6}} & 17.0 & \textcolor{bestgreen}{\textbf{23.0}} & 64.1 & \textcolor{bestgreen}{\textbf{72.1}} & 16.0 & \textcolor{bestgreen}{\textbf{28.0}} & 16.0 & \textcolor{bestgreen}{\textbf{28.0}} \\ 
    \midrule
\multirow{5}{*}{OSS} 
    & Llama2-7B Chat        & 0.0 & \textcolor{bestgreen}{\textbf{8.4\  }}  & 0.0 & \textcolor{bestgreen}{\textbf{10.2}} & 0.0 & \textcolor{bestgreen}{\textbf{25.0}} & 0.0 & \textcolor{bestgreen}{\textbf{5.0\ \ }}  & 4.4 & \textcolor{bestgreen}{\textbf{10.4}} & 0.0 & \textcolor{bestgreen}{\textbf{15.0}} & 0.0 & \textcolor{bestgreen}{\textbf{15.0}} \\ 
    & CodeLlama-7B Instruct & 5.7 & \textcolor{bestgreen}{\textbf{18.4}} & 2.6 & \textcolor{bestgreen}{\textbf{19.2}} & 0.0 & \textcolor{bestgreen}{\textbf{27.0}} & 0.0 & \textcolor{bestgreen}{\textbf{12.5}} & 16.3 & \textcolor{bestgreen}{\textbf{40.2}} & 0.0 & \textcolor{bestgreen}{\textbf{15.0}} & 15.0 & 15.0 \\ 
    & Qwen1.8B Chat         & 2.7 & \textcolor{bestgreen}{\textbf{18.7}} & 1.4 & \textcolor{bestgreen}{\textbf{15.1}} & 6.8 & \textcolor{bestgreen}{\textbf{45.3}} & 0.0 & \textcolor{bestgreen}{\textbf{10.5}} & 6.6 & \textcolor{bestgreen}{\textbf{11.4}} & 0.6 & \textcolor{bestgreen}{\textbf{13.6}} & 13.6 & 13.6 \\ 
    & Qwen-7B Chat          & 5.6 & \textcolor{bestgreen}{\textbf{22.2}} & 4.8 & \textcolor{bestgreen}{\textbf{18.0}} & 0.0 & \textcolor{bestgreen}{\textbf{48.0}} & 34.0 & \textcolor{bestgreen}{\textbf{38.5}} & 0.0 & \textcolor{bestgreen}{\textbf{13.6}} & 0.0 & \textcolor{bestgreen}{\textbf{15.0}} & 15.0 & 15.0 \\ 
    & ChatGLM2-6B v1.1      & 0.0 & \textcolor{bestgreen}{\textbf{15.2}} & 0.0 & \textcolor{bestgreen}{\textbf{16.3}} & 0.0 & \textcolor{bestgreen}{\textbf{17.0}} & 0.0 & \textcolor{bestgreen}{\textbf{5.0\ \ }}  & 0.3 & \textcolor{bestgreen}{\textbf{10.3}} & 4.9 & \textcolor{bestgreen}{\textbf{14.9}} & 14.9 & 14.9 \\    
    \bottomrule
  \end{tabular}
  }
  \hspace*{-0.6cm}  % 与上面值一致，保持对称
\end{table*}

\vspace{-0.5em}
\subsection{Complex Problem-Solving Tasks Evaluation}
\vspace{-0.5em}
We evaluated SAGE against baseline models on three tasks: long-form QA~\cite{akash2023longformquestionansweringiterative}, multi-turn dialog~\cite{cui-etal-2020-mutual}, and sequential task completion~\cite{stephens2012simultaneity}. As shown in Table~\ref{tab:5}, SAGE outperforms all baselines with significant gains, such as a 20.8\% increase in answer accuracy for GPT-3.5 on HotpotQA~\cite{yang2018hotpotqadatasetdiverseexplainable} and a 17.3\% improvement in task completion for Mistral-7b on ALFWorld~\cite{ALFWorld20}. Across all tasks, SAGE notably enhances answer accuracy, dialog coherence, and step completion.

\begin{table*}[!ht]
\centering
\caption{Evaluation of SAGE and Baseline Models on Three Different Tasks}
\renewcommand{\arraystretch}{1.2}
\setlength{\tabcolsep}{4pt}
\small 
  \hspace*{-1.0cm} 
    \resizebox{1.1\textwidth}{!}{  % 1.25 表示宽度略大于页面宽度，保持左右超出一致
\begin{tabular}{p{3.5cm}cccc}
\toprule
\multirow{2}*{\textbf{Agent}} & \textbf{Task Completion} & \textbf{Answer Accuracy} & \textbf{Dialog Coherence} & \textbf{Step Completion} \\
& \textbf{Time (min)} & \textbf{(QA)(\%)} & \textbf{(\%)} & \textbf{Accuracy (\%)} \\
\midrule
\rowcolor{rowgray!20}
GPT-3.5 (Baseline)       & Long-form QA (HotpotQA)               & 54.1\%                    & 48.5\%                     & 62.7\% \\
GPT-4 (Baseline)         & Long-form QA (HotpotQA)               & 61.2\%                    & 53.8\%                     & 68.2\% \\
Llama2-7b (Baseline)     & Multi-turn Dialog (MultiWOZ)          & 55.9\%                    & 50.1\%                     & 64.8\% \\
Codellama-13b (Baseline) & Multi-turn Dialog (MultiWOZ)          & 58.4\%                    & 52.3\%                     & 66.7\% \\
Mistral-7b (Baseline)    & Sequential Task (ALFWorld)            & 56.5\%                    & 51.5\%                     & 65.1\% \\
\midrule
\rowcolor{rowgray!20}
\textbf{SAGE-GPT-3.5}             & Long-form QA (HotpotQA)               & \textbf{\textcolor{bestgreen}{74.9\% (+20.8\%)}} & \textbf{\textcolor{bestgreen}{68.3\% (+19.8\%)}} & \textbf{\textcolor{bestgreen}{80.6\% (+17.9\%)}} \\
\textbf{SAGE-GPT-4}               & Long-form QA (HotpotQA)               & \textbf{\textcolor{bestgreen}{78.4\% (+17.2\%)}} & \textbf{\textcolor{bestgreen}{73.4\% (+19.6\%)}} & \textbf{\textcolor{bestgreen}{83.9\% (+15.7\%)}} \\
\textbf{SAGE-Llama2-7b}           & Multi-turn Dialog (MultiWOZ)          & \textbf{\textcolor{bestgreen}{72.2\% (+16.1\%)}} & \textbf{\textcolor{bestgreen}{67.9\% (+17.8\%)}} & \textbf{\textcolor{bestgreen}{78.5\% (+13.7\%)}} \\
\textbf{SAGE-Codellama-13b}       & Multi-turn Dialog (MultiWOZ)          & \textbf{\textcolor{bestgreen}{74.7\% (+16.3\%)}} & \textbf{\textcolor{bestgreen}{71.2\% (+18.9\%)}} & \textbf{\textcolor{bestgreen}{81.2\% (+14.5\%)}} \\
\textbf{SAGE-Mistral-7b}          & Sequential Task (ALFWorld)            & \textbf{\textcolor{bestgreen}{73.8\% (+17.3\%)}} & \textbf{\textcolor{bestgreen}{70.5\% (+19.0\%)}} & \textbf{\textcolor{bestgreen}{79.9\% (+14.8\%)}} \\
\bottomrule
\end{tabular}}
  \hspace*{-1.0cm} 
\label{tab:5}
\end{table*}

\vspace{-0.5em}
\subsection{Evaluation of Long-Context Tasks}
\vspace{-0.5em}
We evaluated the agent's code generation and reasoning on four long-text tasks: \textbf{LCC Dataset}~\cite{mohler-etal-2016-introducing} focuses on predicting the next line of code from a few initial lines, with Precision, Recall, and F1 as metrics. \textbf{RepoBench-P}~\cite{liu2024repobench} tests retrieval of relevant code snippets from cross-file and within-file contexts to predict the next line, also evaluated with Precision, Recall, and F1. \textbf{HotPotQA}~\cite{yang-etal-2018-hotpotqa}, a Wikipedia-based dataset with 113k question-answer pairs, challenges the agent to reason across multiple documents, evaluated by answer F1. \textbf{TriviaQA}~\cite{joshi-etal-2017-triviaqa} is a reading comprehension dataset with question-answer pairs and evidence paragraphs (filtered to over 1,000 words), also using answer F1 for evaluation.

We compared two self-refinement methods: \textbf{Beam Search}~\cite{kool2019stochasticbeamsthemgumbeltopk}, which integrates self-assessment through stochastic beam search, and \textbf{Reflexion}~\cite{shinn2023reflexion}, which uses past trial experience in a verbal form.

\vspace{-0.5em}
\subsection{Evaluation Results:}
\vspace{-0.5em}
\textbf{Code Completion Task:}  
On the LCC dataset (Table~\ref{tab:3}), SAGE shows a slight improvement in F1 score (79.29) compared to Beam Search and Reflexion. Its memory mechanisms help refine code predictions, but the performance difference is not substantial in simpler tasks like code completion.

\textbf{Reasoning Tasks:}  
SAGE significantly outperforms Reflexion and Beam Search on HotPotQA and TriviaQA, with F1 scores of 22.06 and 22.76 (Table~\ref{tab:3}). SAGE's ability to effectively integrate multi-document information through reflection leads to better reasoning accuracy, while Reflexion and Beam Search face challenges in handling complex reasoning tasks.

\begin{table*}[!ht]
\centering
\caption{Comparison of Performance Across Different Methods}
  \resizebox{0.9\textwidth}{!}{%
\begin{tabular}{lccccccccc}  
\toprule
\multirow{2}*{\textbf{Models} }& \multicolumn{3}{c}{\textbf{LCC}} & \multicolumn{3}{c}{\textbf{RepoBench-P}} & \textbf{HotpotQA} & \textbf{TriviaQA} \\
\cmidrule(r){2-4} \cmidrule(r){5-7} \cmidrule(r){8-8} \cmidrule(r){9-9}
 & \textbf{Precision} & \textbf{Recall} & \textbf{F1} & \textbf{Precision} & \textbf{Recall} & \textbf{F1} & \textbf{F1} & \textbf{F1} \\
\midrule
Reflexion & 77.72 & \cellcolor{blue!15}81.00 & 79.28 & 78.73 & 81.86 & 80.25 & 11.26 & 11.23 \\
Beam search & \cellcolor{blue!15}78.98 & 79.32 & 79.12 & 78.75 & 81.02 & 79.87 & 10.26 & 12.13 \\
SAGE & 78.76 & 79.88 & \cellcolor{blue!15}79.29 & \cellcolor{blue!15}79.27 & \cellcolor{blue!15}83.28 & \cellcolor{blue!15}81.22 & \cellcolor{red!15}22.06 & \cellcolor{red!15}{22.76} \\
\bottomrule
\end{tabular}}
\label{tab:3}
\end{table*}

\vspace{-0.5em}
\subsection{Evaluation of RAG Agents}
\vspace{-0.5em}
Table \ref{tab:rag_evaluation} compares classical lexical retrieval (RAG with BM25)~\cite{INR-019}, dense passage retrieval (RAG with DPR)~\cite{reichman2024densepassageretrievalretrieving}, a commercial retrieval solution (RAG with OpenAI Retrieval)~\cite{openai_cookbook}, the TART~\cite{eisenschlos2022tart} sequence-to-sequence retrieval model, and the FiD (Fusion-in-Decoder) method~\cite{izacard2021fid}, all tested on multi-document and open-domain QA tasks (HotpotQA, Natural Questions~\cite{kwiatkowski-etal-2019-natural}, and TriviaQA). RAG with BM25 relies on term-based matching, while RAG with DPR uses learned dense embeddings. TART adopts a transformer-based approach to produce relevant contexts, and FiD fuses multiple retrieved passages through an encoder-decoder design. In contrast, ChatGPT-4 (SAGE) employs a structured reasoning workflow for retrieval and generation, which leads to steady accuracy improvements of 3.6\% to 4.7\% and cuts memory consumption nearly 50\% on some tasks, all without increasing latency.

\begin{table*}[!ht]
    \centering
    \caption{Evaluation of different RAG Agents on Different Tasks and Datasets}
    \label{tab:rag_evaluation}
    \renewcommand{\arraystretch}{1.3}
    \setlength{\tabcolsep}{6pt} 
    \small 
      \resizebox{0.9\textwidth}{!}{%
    \begin{tabularx}{\textwidth}{Xccc}
        \toprule
        \textbf{Agent} & \textbf{Accuracy (QA) (\%)} & \textbf{Latency (ms)} & \textbf{Memory Usage (MB)} \\ 
        \midrule
        \multicolumn{4}{c}{\textbf{Task 1: Multi-Document QA (HotpotQA)}} \\
        RAG (BM25)                & 60.8 & 121 & 613 \\
        RAG (DPR)                 & 66.3 & 129 & 542 \\
        RAG (OpenAI Retrieval)    & 67.4 & 108 & 494 \\
        TART                      & 63.2 & 144 & 477 \\
        FiD (Fusion-in-Decoder)    & 70.1 & 153 & 456 \\
        \rowcolor{rowgray!10}
        \textbf{ChatGPT-4 - Sage} & \textbf{\textcolor{bestgreen}{74.8 (+4.7)}} & 128 & \textbf{\textcolor{bestgreen}{231 (-50\%)}} \\ 
        \midrule
        \multicolumn{4}{c}{\textbf{Task 2: Document Retrieval for Contextual Answering (Natural Questions)}} \\
        RAG (BM25)                & 59.9 & 125 & 605 \\
        RAG (DPR)                 & 65.5 & 131 & 561 \\
        RAG (OpenAI Retrieval)    & 66.8 & 113 & 484 \\
        TART                      & 62.4 & 146 & 455 \\
        FiD (Fusion-in-Decoder)    & 69.8 & 156 & 443 \\
        \rowcolor{rowgray!10}
        \textbf{ChatGPT-4 - Sage} & \textbf{\textcolor{bestgreen}{73.6 (+3.8)}} & 131 & \textbf{\textcolor{bestgreen}{227 (-49\%)}} \\ 
        \midrule
        \multicolumn{4}{c}{\textbf{Task 3: Open-Domain QA with Multiple Contexts (TriviaQA)}} \\
        RAG (BM25)                & 62.1 & 124 & 615 \\
        RAG (DPR)                 & 67.8 & 129 & 530 \\
        RAG (OpenAI Retrieval)    & 68.9 & 117 & 494 \\
        TART                      & 64.7 & 148 & 462 \\
        FiD (Fusion-in-Decoder)    & 71.9 & 155 & 456 \\
        \rowcolor{rowgray!10}
        \textbf{ChatGPT-4 - Sage} & \textbf{\textcolor{bestgreen}{75.5 (+3.6)}} & 134 & \textbf{\textcolor{bestgreen}{243 (-47\%)}} \\ 
        \bottomrule
    \end{tabularx}}
\end{table*}

\vspace{-0.5em}
\subsection{Error analysis}
\vspace{-0.5em}
As shown in Figure~\ref{fig:3}, the SAGE framework significantly enhances agent performance across tasks, especially in the WS task for AgentBench, due to its iterative feedback mechanism, which refines outputs through continuous assistant-checker interaction. In OS and DB tasks, Context Limit Exceeded and invalid format errors are nearly eliminated, with a notable reduction in invalid action errors, attributed to the reflection mechanism that helps the assistant learn and reduce logical mistakes.

\begin{figure*}[!ht]
    \centering
    \includegraphics[width=1\linewidth]{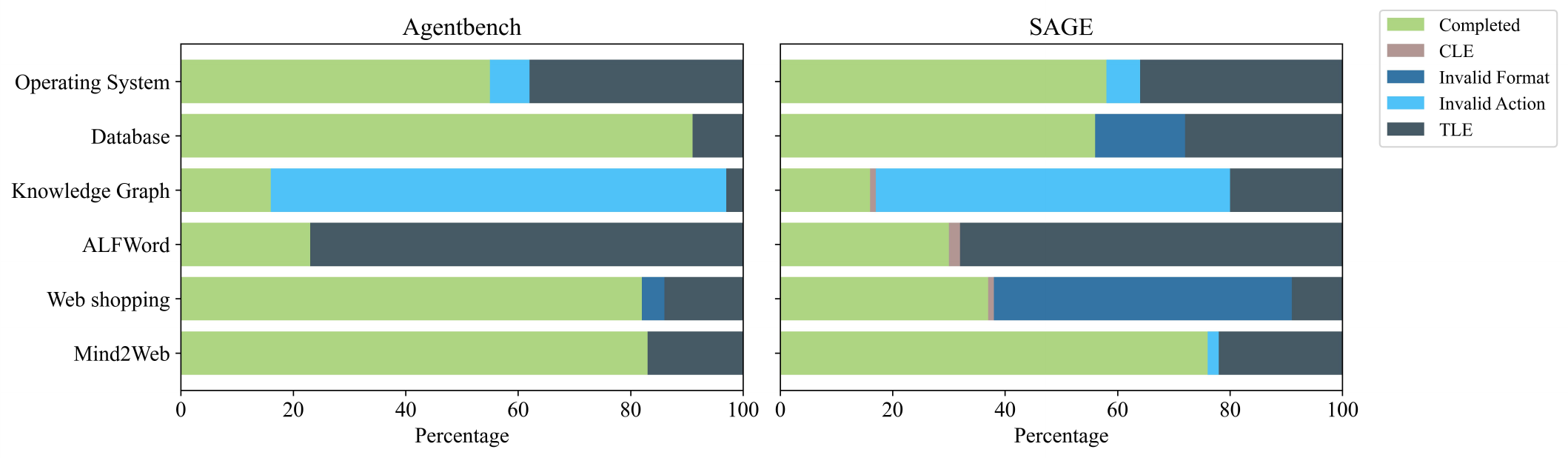}
    \caption{Execution results across six tasks (CLE: Context Limit Exceeded, TLE: Task Limit Exceeded). Task limits are the main cause of incomplete tasks, highlighting LLM agents' limitations under time constraints.}
    \label{fig:3}
    \vspace{-10pt}
\end{figure*}

\vspace{-0.5em}
\subsection{Ablation Study}
\vspace{-0.5em}
We conducted ablation experiments on Qwen-1.8B and CodeLlama-7B to evaluate memory optimization (Table \ref{tab:4}). Without memory optimization, both models perform weakly, especially Qwen-1.8B, which improves from 6.8 to 48.0 in KG and from 0.0 to 10.5 in ALF after optimization. Similarly, CodeLlama-7B shows substantial gains, particularly in DB (2.7 to 41.3) and WS (14.3 to 58.7). Overall, CodeLlama-7B performs better than Qwen-1.8B, highlighting the stronger adaptability of models with more parameters in handling complex tasks.

\begin{table}[!ht]
    \centering
    \caption{Ablation study for memory optimization on the task of AgentBench}
    \resizebox{0.7\linewidth}{!}{
    \begin{tabular}{ccccccc}
        \toprule
        \textbf{Models} & \textbf{OS} & \textbf{DB} & \textbf{KG} & \textbf{ALF} & \textbf{WS} & \textbf{M2W} \\
        \midrule
        \textbf{Qwen-1.8B (w/o memo)} & 10.4 & 22.6 & {6.8 \ } & {0.0\ \ }  & {26.6 \ }& {5.0\ \ }  \\
        \textbf{Qwen-1.8B (w memo)} & \textcolor{bestgreen}{\textbf{18.7}} & \textcolor{bestgreen}{\textbf{28.3}} & \textcolor{bestgreen}{\textbf{45.3} }& \textcolor{bestgreen}{\textbf{10.5} }& \textcolor{bestgreen}{\textbf{31.4} }& \textcolor{bestgreen}{\textbf{25.1} }\\
        \midrule
        \textbf{Codellama-7B (w/o memo)} & {9.7\ } & {2.7\ } & {0.0 \ \ }& {0.0\ \ } & {14.3\ } & {5.0\ \ } \\
        \textbf{Codellama-7B (w memo)} & \textcolor{bestgreen}{\textbf{23.4}} & \textcolor{bestgreen}{\textbf{41.3}} & \textcolor{bestgreen}{\textbf{48.0} }& \textcolor{bestgreen}{\textbf{12.5} }& \textcolor{bestgreen}{\textbf{58.7} }& \textcolor{bestgreen}{\textbf{15.0} }\\
        \bottomrule
    \end{tabular}
    }
    \label{tab:4}
    \vspace{-10pt}
\end{table}

\vspace{-0.5em}
\section{Conclusion}
\vspace{-1em}
In this paper, we propose the SAGE framework, which enhances agents' self-adjustment and memory management in complex tasks through reflective mechanisms and memory optimization. Experimental results show significant performance improvements across benchmarks, especially in smaller models. In the AgentBench test, SAGE boosts the performance of strong baselines like GPT-3.5 and GPT-4, while also significantly improving open-source models. It effectively reduces basic errors and logical mistakes, particularly enabling smaller models to handle complex tasks.

\bibliographystyle{colm2025_conference}
% \section{Limitations}
% Despite the significant improvements achieved by the SAGE framework, several limitations remain. The iterative feedback process can become computationally intensive, leading to increased latency and higher resource consumption, which may not be feasible for real-time applications or systems with limited computational power. Additionally, managing short-term and long-term memory introduces complexity and overhead, requiring further optimization. The framework's performance heavily relies on the quality and accuracy of the checker; incorrect feedback can lead to suboptimal solutions. While the SAGE framework has shown effectiveness on specific benchmark tasks, its generalization across a broader range of tasks needs validation. It may struggle with tasks involving high levels of ambiguity or requiring deep contextual understanding. The MemorySyntax method's thresholds for retaining or forgetting information are critical and can affect performance, necessitating careful tuning. The use of sparse reward signals may not provide sufficient granularity for learning, requiring more nuanced feedback mechanisms. Finally, the evaluation of the SAGE framework has primarily been conducted on specific datasets, introducing a risk of evaluation bias. Broader, more diverse evaluations are needed to ensure robustness and general applicability. Addressing these limitations will require further research and development to refine the framework, optimize its components, and validate its performance across diverse and complex tasks.
\bibliography{custom}

\appendix

\begin{figure*}[!ht]
    \centering
    \includegraphics[width=\textwidth]{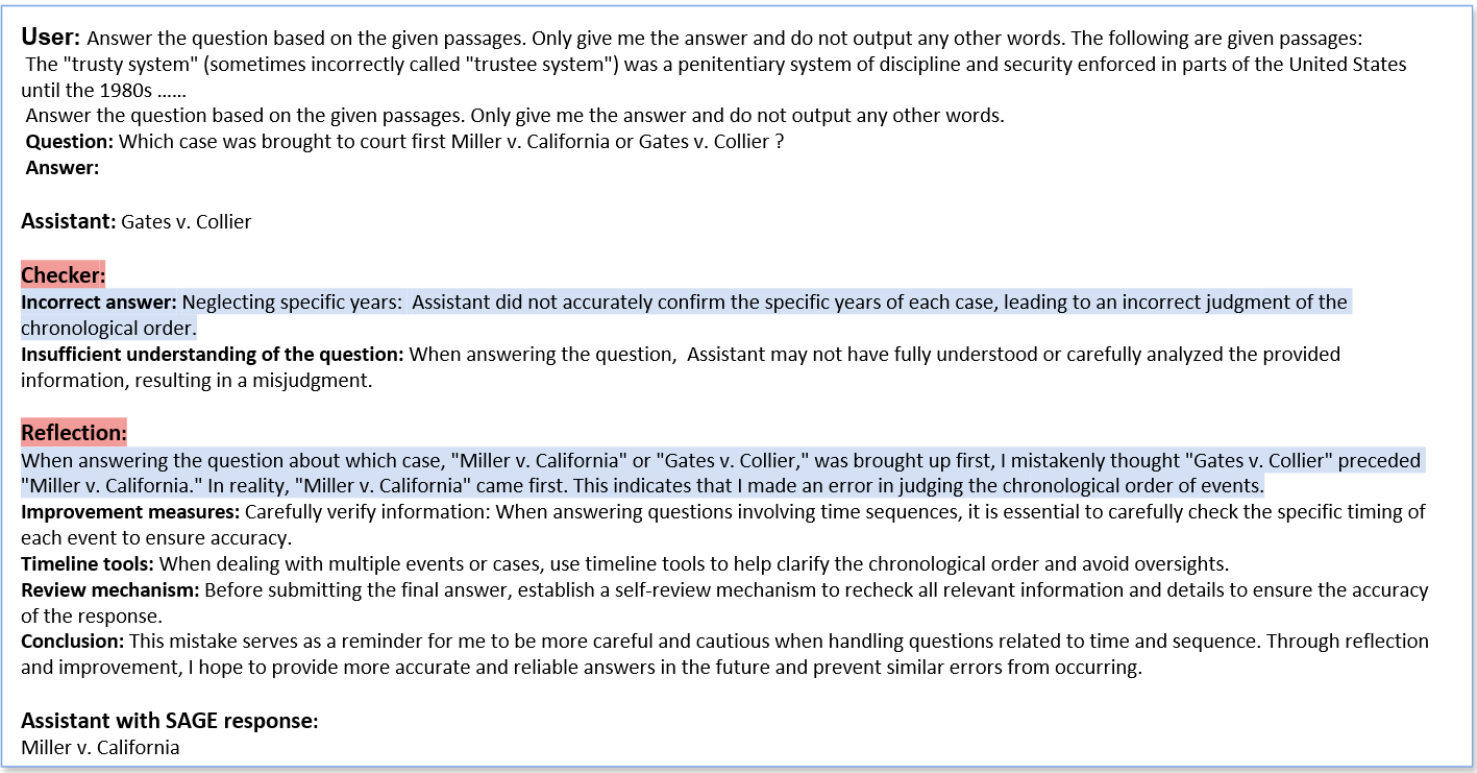}
    \caption{The illustration of an example HotpotQA with SAGE.}
    \label{fig:2}
\end{figure*}

\section{Theoretical Analysis}
\subsection{An Information-Theoretic Perspective on Memory Optimization Based on the Ebbinghaus Forgetting Curve}\label{sec:thm1}

The memory optimization mechanism proposed in this work is grounded in the Ebbinghaus forgetting curve. It dynamically manages short-term memory (STM) and long-term memory (LTM) by prioritizing the retention of essential information while discarding irrelevant data. This mechanism can be viewed as computing the retention strength \( S(I_t) \) of each information unit based on its information entropy. Theoretically, one can interpret this as maximizing information entropy to show that the mechanism effectively mitigates the issue of information overload~\cite{10.5555/1146355}. That is, by preferentially retaining high-entropy (i.e., high-uncertainty) information, the model improves its information processing efficiency.

First, the retention strength of information is defined as:
\begin{equation}
S(I_t) = \frac{H(I_t)}{f(t)}
\end{equation}

where the numerator \( H(I_t) \) represents the information entropy of the content at time \( t \), and \( f(t) \) models the forgetting curve.

The objective of memory optimization is to maximize the utility of retained information under a memory capacity constraint:

\begin{equation}
\max \sum_t M_t \cdot S(I_t), \quad \text{s.t.} \quad \sum_t M_t \leq C
\end{equation}

Here, \( C \) represents the total memory capacity, and \( M_t \in \{0,1\} \) denotes whether the information at time \( t \) is retained (1) or discarded (0).

Using the method of Lagrange multipliers, we construct the Lagrangian function:

\begin{equation}
\mathcal{L}(M_t, \lambda) = \sum_t M_t \cdot S(I_t) - \lambda \left( \sum_t M_t - C \right)
\end{equation}

Taking the derivative of \( \mathcal{L} \), which is the optimal memory retnetion, with respect to \( M_t \) and setting it to zero yields the optimality condition:

\begin{equation}
\frac{\partial \mathcal{L}}{\partial M_t} = S(I_t) - \lambda = 0 \quad \Rightarrow \quad
\begin{cases}
\text{Retain } I_t & \text{if } S(I_t) \geq \lambda \\
\text{Discard } I_t & \text{if } S(I_t) < \lambda
\end{cases}
\end{equation}

This rule is consistent with the threshold-based memory retention strategy used in the SAGE framework.

\subsection{Iterative Feedback and Game-Theoretic Model}\label{sec:thm2}
Multi-agent systems often depend on reinforcement learning (RL) methods, making their theoretical basis essential for ensuring that algorithms converge to stable solutions, such as Nash equilibria~\cite{fudenberg1991game}. The subsequent theory provides a formal guarantee of convergence, showing that our proposed algorithm is not randomly designed. Instead, it surpasses standard approaches by establishing robust criteria that confirm both its stability and performance advantages.

\subsubsection{Players and Their Strategy Spaces Settings}

\begin{equation}
\begin{array}{rcl}
\textbf{User (U)} &:& \text{Provides a task description } T_U \text{ and constraints } C_U. \\[1mm]
\textbf{Assistant (A)} &:& \text{Chooses actions } a_t \sim \pi_\theta(\cdot \mid s_t), \text{ where } \pi_\theta \text{ is the policy}. \\[1mm]
\textbf{Checker (C)} &:& \text{Generates feedback } f_t \text{ aimed at reducing the error } \|o_t - o^*\|_F,
\end{array}
\end{equation}
where $o_t$ is the output of the assistant and $o^*$ is the expected (ground-truth) output, $\|\cdot \|_F$ denotes the Frobenius norm.

For the Utility Functions, we have:
\begin{equation}
\begin{array}{rcl}
\textbf{Assistant’s Utility, } R_A &:& \text{Measures task performance (e.g., BLEU score, accuracy):} \\
&& R_A = \text{Task Performance Metric of } a_t. \\[1mm]
\textbf{Checker’s Utility, } R_C &:& \text{Measures feedback accuracy by reduction of error:} \\
&& R_C \propto \Delta \|o_t - o^*\|_F.
\end{array}
\end{equation}

The overall interaction can be viewed as a three-player game in which:
The assistant (A) optimizes its policy $ \pi_\theta$ to maximize $R_A$, while the checker (C) aims to minimize  $\|o_t - o^*\|$ (i.e., maximize  $R_C$) The user U sets the task and constraints that guide both the assistant and checker’s performances. Next we want to show that it can be shown that under a finite strategy space and continuous utility functions, a Nash equilibrium exists in this multi-agent game.

\subsubsection{Setup for the proof of theorem~\ref{thm:nash_stability}}

First, we want to write a road map for this proof.
\begin{itemize}[leftmargin=10pt]
    \item It can be shown that under a finite strategy space and continuous utility functions, a Nash equilibrium exists in this multi-agent game.
    \item According to the \textbf{Debreu-Glicksberg-Fan Theorem} (a generalization of Nash's existence theorem), if each player’s strategy space is a compact convex set and the utility functions are continuous and quasi-concave (or concave) in their own strategies, then at least one Nash equilibrium exists.
    \item Then, the assistant’s policy \( \pi_\theta \) is updated using policy gradient methods. Since \( R_A \) is a continuous function of \( \pi_\theta \), the assistant’s utility is continuous.  
    \item After that, the checker’s feedback strategy can be obtained via convex optimization (e.g., least-squares regression), and thus \( R_C \) is a convex function.  
    \item Therefore, a Nash equilibrium \( (\pi_\theta^*, f^*) \) exists, at which both the assistant and checker’s strategies are stable with respect to each other:
\begin{equation}
\pi_\theta^* = \arg\max_{\pi} R_A(\pi, f^*), \quad
f^* = \arg\max_f R_C(\pi_\theta^*, f)
\end{equation}
    
\end{itemize}

Let's start the proof.

\begin{lemma}\label{lemma:1}
Let $\mathcal{S}_1$, $\mathcal{S}_2, \dots, \mathcal{S}_N$ be finite strategy sets for each of  $N$ players, and let  $u_i: \mathcal{S}_1 \times \cdots \times \mathcal{S}_N \to \mathbb{R} $ be continuous for each  $i$. Then, there exists a Nash equilibrium in the multi-agent game.
\end{lemma}

\begin{proof}

For each player $ i $, define the set of mixed strategies as
\begin{equation}
\Delta(\mathcal{S}_i)=\left\{ x_i\in \mathbb{R}^{|\mathcal{S}_i|} \ \Big|\ x_i(s)\ge0 \text{ for all } s\in \mathcal{S}_i, \quad \sum_{s\in \mathcal{S}_i} x_i(s)=1 \right\}.
\end{equation}
Each set $ \Delta(\mathcal{S}_i) $ is a simplex in $ \mathbb{R}^{|\mathcal{S}_i|} $, and hence is nonempty, compact, and convex. Define the expected utility function for player $ i $:
\begin{equation}
U_i(x_1,\dots,x_N)= \sum_{(s_1,\dots,s_N) \in \mathcal{S}_1 \times \cdots \times \mathcal{S}_N} \left( \prod_{j=1}^N x_j(s_j) \right) u_i(s_1,\dots,s_N).
\end{equation}
Since the strategy spaces are finite and $ u_i $ is continuous (and hence bounded on the finite set), the function $ U_i $ is continuous on the product space $\prod_{j=1}^N \Delta(\mathcal{S}_j)$.

For each player $ i $, define the best response correspondence as:
\begin{equation}
BR_i(x_{-i})=\arg\max_{x_i\in \Delta(\mathcal{S}_i)} U_i(x_i,x_{-i}),
\end{equation}
where $ x_{-i} $ denotes the mixed strategies of all players other than $ i $. Because $ \Delta(\mathcal{S}_i) $ is compact and $ U_i $ is continuous and linear in $ x_i $ (hence concave), the maximum exists and $ BR_i(x_{-i}) $ is nonempty, convex, and compact. Now we define the overall best-response correspondence:
\begin{equation}
BR(x)= BR_1(x_{-1}) \times BR_2(x_{-2}) \times \cdots \times BR_N(x_{-N}),
\end{equation}
where $ x=(x_1,x_2,\dots,x_N) $. The product space $\prod_{i=1}^N \Delta(\mathcal{S}_i)$ is compact and convex, and the correspondence $ BR $ has nonempty, convex, and compact values and is upper hemicontinuous (due to the continuity and linearity of the expected utilities). Hence, by Kakutani's fixed point theorem, there exists a fixed point 
\begin{equation}
x^*=(x_1^*, x_2^*, \dots, x_N^*)\in \prod_{i=1}^N \Delta(\mathcal{S}_i)
\end{equation}
such that
\begin{equation}
x^* \in BR(x^*).
\end{equation}
This fixed point is a Nash equilibrium in mixed strategies.

Thus, under a finite strategy space and continuous utility functions, a Nash equilibrium exists. According to the \textbf{Debreu-Glicksberg-Fan Theorem} (a generalization of Nash's existence theorem), if each player’s strategy space is a compact convex set and the utility functions are continuous and quasi-concave (or concave) in their own strategies, then at least one Nash equilibrium exists.

\end{proof}

\begin{lemma}\label{lemma:2}
By Lemma~\ref{lemma:1}, Let $I = \{1,2,\dots,n\}$ be the set of players. For each player $i$, let  $X_i \subset \mathbb{R}^{d_i}$ be a compact, convex strategy set, and define  $X = \prod_{i=1}^n X_i$. Suppose each player’s utility function $u_i: X \to \mathbb{R}$ is continuous on $X$ and quasi-concave (or concave) in the coordinate  $x_i$ for fixed $x_{-i}$. Then, by the Debreu-Glicksberg-Fan theorem, there exists at least one Nash equilibrium $x^* = (x_1^*, x_2^*, \dots, x_n^*) \in X$.
\end{lemma}

\begin{proof}
For each player $ i $, define the best-response correspondence
\begin{equation}
BR_i(x_{-i}) = \{ x_i \in X_i : u_i(x_i, x_{-i}) = \max_{y_i \in X_i} u_i(y_i, x_{-i}) \}.
\end{equation}
Since $ u_i $ is continuous and $ X_i $ is compact, the maximum is attained; hence, $ BR_i(x_{-i}) \neq \varnothing $. The quasi-concavity (or concavity) of $ u_i $ in $ x_i $ ensures that the set 
\begin{equation}
\{ x_i \in X_i : u_i(x_i, x_{-i}) \geq \alpha \}
\end{equation}
is convex for any $ \alpha \in \mathbb{R} $. In particular, the set $ BR_i(x_{-i}) $ is convex for every fixed $ x_{-i} $.

The continuity of $ u_i $ implies that the best-response correspondence $ BR_i $ is upper hemicontinuous. This means that for any sequence $ \{x_{-i}^{(k)}\} $ converging to $ x_{-i} $ and any corresponding sequence $ \{ x_i^{(k)} \} $ with $ x_i^{(k)} \in BR_i(x_{-i}^{(k)}) $, every limit point $ x_i $ of $ \{ x_i^{(k)} \} $ belongs to $ BR_i(x_{-i}) $.

Now we define the aggregate best-response correspondence:
\begin{equation}
BR(x) = BR_1(x_{-1}) \times BR_2(x_{-2}) \times \cdots \times BR_n(x_{-n}),
\end{equation}
where $ x = (x_1, x_2, \dots, x_n) \in X $. Note that: (1) $X$ is a nonempty, compact, convex subset of $ \mathbb{R}^{d_1 + d_2 + \dots + d_n} $. (2) For each $ i $, $ BR_i(x_{-i}) $ is nonempty and convex. (3) The correspondence $ BR $ is upper hemicontinuous.

By the Debreu-Glicksberg-Fan fixed point theorem, there exists a fixed point $ x^* \in X $ such that $x^* \in BR(x^*)$. That is, for every $ i \in I $,$x_i^* \in BR_i(x_{-i}^*)$. This means that no player can unilaterally deviate from $ x^* $ and increase their utility, which is precisely the definition of a Nash equilibrium.

\end{proof}

Next, based on the above Lemmas, we want to prove that the assistant’s policy \( \pi_\theta \) is updated using policy gradient methods~\cite{sutton2018reinforcement}. Since \( R_A \) is a continuous function of \( \pi_\theta \), the assistant’s utility is continuous.  

\begin{lemma}\label{lemma:3}
Let $\Theta$ denote the parameter space for the assistant's policy $\pi_{\theta}$, and let the assistant's utility be given by
\begin{equation}
R_A(\pi_\theta) = f\big(\pi_\theta\big),
\end{equation}
where $f: \Pi \to \mathbb{R}$ is a continuous function on the space $\Pi$ of policies. Assume that the mapping $\theta \mapsto \pi_\theta$ is continuous. Then the composite function
\begin{equation}
F(\theta) = f\big(\pi_\theta\big)
\end{equation}
is continuous in $\theta$. Moreover, since the assistant's policy is updated using policy gradient methods, the update rule
\begin{equation}
\theta_{k+1} = \theta_k + \alpha_k \nabla_{\theta}  F(\theta_k)
\end{equation}
employs the gradient of a continuous function.
\end{lemma}

\begin{proof}

We divide the proof into two parts:

Step 1: Continuity of the Utility Function. Since the mapping $\theta \mapsto \pi_\theta$ is continuous and the function $f:\Pi \to \mathbb{R}$ is continuous by hypothesis, their composition $F(\theta) = f\big(\pi_\theta\big)$ is continuous in $\theta$. Formally, for any $\epsilon > 0$, there exists a $\delta > 0$ such that for all $\theta, \theta^\prime \in \Theta$ satisfying $\|\theta - \theta^\prime\| < \delta$,
\begin{equation}
|F(\theta) - F(\theta^\prime)| = \big|f\big(\pi_\theta\big) - f\big(\pi_{\theta^\prime}\big)\big| < \epsilon.
\end{equation}

Step 2: Policy Gradient Updates. In the policy gradient framework, the assistant updates its parameters according to the rule:
\begin{equation}
\theta_{k+1} = \theta_k + \alpha_k \nabla_{\theta}  F(\theta_k),
\end{equation}
where $\alpha_k > 0$ is a (possibly diminishing) step size and $\nabla_{\theta}  F(\theta_k)$ denotes the gradient of $F$ evaluated at $\theta_k$. Since $F$ is continuous (and under typical smoothness assumptions, differentiable with continuous gradient), the gradient $\nabla_{\theta}  F$ is well-defined. The policy gradient update is thus a continuous map from $\theta_k$ to $\theta_{k+1}$.

Because the update at each iteration is based on the gradient (which is continuous) of the assistant's utility $F(\theta)$, any small perturbation in $\theta$ results in only a small variation in both the gradient and the updated parameter. This demonstrates that the update procedure is stable in the sense that the assistant's utility, $R_A(\pi_\theta) = F(\theta)$, continuously varies under the sequence of policy gradient updates~\cite{sutton2018reinforcement}.

Thus, since $R_A$ is a continuous function of $ \pi_\theta$ and $\pi_\theta$ depends continuously on  $\theta$, the assistant's utility is continuous, and the policy gradient update is well-defined.
\end{proof}

Then we want to prove that, the checker’s feedback strategy can be obtained via convex optimization (e.g., least-squares regression), and thus \( R_C \) is a convex function.  

\begin{lemma}\label{lemma:4}
Let $\mathcal{F}$ be a convex subset of a vector space, and suppose the checker’s feedback strategy is given by a function $f\in\mathcal{F}$ obtained as a solution to the convex optimization problem
\begin{equation}
\min_{f\in\mathcal{F}} \quad J(f) = \sum_{i=1}^n \big(y_i - f(x_i)\big)^2,
\end{equation}
where $x_i$ are given inputs and $y_i$ are target outputs. Define the checker’s utility function
\begin{equation}
R_C(f) = J(f).
\end{equation}
Then $R_C$ is a convex function on $\mathcal{F}$; that is, for any $f_1, f_2 \in \mathcal{F}$ and any $\lambda \in [0,1]$ it holds that
\begin{equation}
R_C\big(\lambda f_1 + (1-\lambda) f_2\big) \le \lambda R_C(f_1) + (1-\lambda) R_C(f_2).
\end{equation}
\end{lemma}

\begin{proof}
We prove the convexity of $R_C$ by verifying that the squared-error loss is a convex function and that convex combinations preserve convexity.

For each data point $i$, consider the function
\begin{equation}
\ell_i(f) = \big(y_i - f(x_i)\big)^2.
\end{equation}
Since for any real number $t$, the function $t^2$ is convex, and the mapping $f\mapsto y_i-f(x_i)$ is affine (and hence convex), the composition $f\mapsto \big(y_i-f(x_i)\big)^2$ is convex. In particular, for any $f_1, f_2\in\mathcal{F}$ and $\lambda\in[0,1]$,
\begin{equation}
\ell_i\big(\lambda f_1 + (1-\lambda) f_2\big) \le \lambda \ell_i(f_1) + (1-\lambda)\ell_i(f_2).
\end{equation}

Since the checker’s overall loss (and hence utility function)
\begin{equation}
R_C(f) = \sum_{i=1}^n \ell_i(f)
\end{equation}
is a finite sum of convex functions, it follows that $R_C$ is convex on $\mathcal{F}$; that is, for any $f_1, f_2 \in \mathcal{F}$ and $\lambda \in [0,1]$,
\begin{equation}
\begin{aligned}
R_C\big(\lambda f_1 + (1-\lambda) f_2\big) &= \sum_{i=1}^n \ell_i\big(\lambda f_1 + (1-\lambda) f_2\big) \\
&\le \lambda \sum_{i=1}^n \ell_i(f_1) + (1-\lambda)\sum_{i=1}^n \ell_i(f_2) \\
&= \lambda R_C(f_1) + (1-\lambda)R_C(f_2).
\end{aligned}
\end{equation}
Thus:
\begin{equation}
R_C(f)=\sum_{i=1}^n \big(y_i-f(x_i)\big)^2 \text{ is convex on } \mathcal{F}.
\end{equation}
\end{proof}

Thus, the checker’s feedback strategy derived via convex optimization (e.g., least-squares regression) ensures that the corresponding utility function $R_C$ is convex.

\subsubsection{The formal proof of theorem~\ref{thm:nash_stability}}
We want to rewrite the theorem~\ref{thm:nash_stability} into the following lemma based on the former lemmas:
\begin{lemma}[theorem~\ref{thm:nash_stability}]
Suppose the assistant’s strategy space  $\Pi$ and the checker’s strategy space $\mathcal{F}$ are nonempty, compact, and convex sets. Assume that: $R_A: \Pi \times \mathcal{F} \to \mathbb{R}$ is continuous in $\pi$ for fixed $f$, $R_C: \Pi \times \mathcal{F} \to \mathbb{R}$ is convex in $f$ for fixed $\pi$. Then there exists a Nash equilibrium $\pi_\theta^*, f^*$ defined by
\begin{equation}
\pi_\theta^* = \arg\max_{\pi\in\Pi} R_A(\pi, f^*) \quad \text{and} \quad f^* = \arg\max_{f\in\mathcal{F}} R_C(\pi_\theta^*, f).
\end{equation}
\end{lemma}
\begin{proof}

For each fixed $ f \in \mathcal{F} $ define the assistant’s best-response set:
\begin{equation}
BR_A(f) = \{ \pi \in \Pi : R_A(\pi, f) = \max_{\pi' \in \Pi} R_A(\pi', f) \}.
\end{equation}
By hypothesis (i) and the compactness of $\Pi$, the maximum is attained and $ BR_A(f) $ is nonempty. (Continuity of $R_A$ assures that small changes in $f$ lead to small changes in $R_A(\cdot,f)$.)

Similarly, for each fixed $\pi \in \Pi$ define the checker’s best-response set:
\begin{equation}
BR_C(\pi) = \{ f \in \mathcal{F} : R_C(\pi, f) = \max_{f' \in \mathcal{F}} R_C(\pi, f') \}.
\end{equation}
Since hypothesis (ii) implies that $R_C$ is convex in $f$ and $\mathcal{F}$ is convex and compact, we conclude that $ BR_C(\pi) $ is nonempty, convex, and closed.

Consider the correspondence
\begin{equation}
H: \Pi \times \mathcal{F} \to \Pi \times \mathcal{F}, \quad H(\pi,f) = BR_A(f) \times BR_C(\pi).
\end{equation}
By the arguments above: For every $ (\pi,f) \in \Pi \times \mathcal{F} $, the set $ H(\pi,f) $ is nonempty. $ BR_A(f) $ is closed (by continuity and compactness) and $ BR_C(\pi) $ is convex (and closed) by hypothesis (ii). Hence, $ H(\pi,f) $ is convex and closed.

Since the domain $ \Pi \times \mathcal{F} $ is nonempty, compact, and convex, and since the correspondence $ H $ has nonempty, convex, and closed values and is upper hemicontinuous (this follows from the continuity of $ R_A $ in $\pi$ and convexity of $ R_C $ in $ f $), Kakutani’s fixed point theorem guarantees the existence of a fixed point
\begin{equation}
(\pi^*, f^*) \in \Pi \times \mathcal{F}\quad\text{such that}\quad (\pi^*, f^*) \in H(\pi^*, f^*).
\end{equation}

By the definition of $ H $, the fixed point satisfies
\begin{equation}
\pi^* \in BR_A(f^*) \quad \text{and} \quad f^* \in BR_C(\pi^*),
\end{equation}
which is equivalent to
\begin{equation}
\pi^* = \arg\max_{\pi\in\Pi} R_A(\pi, f^*) \quad \text{and} \quad f^* = \arg\max_{f\in\mathcal{F}} R_C(\pi^*, f).
\end{equation}

Thus, no unilateral deviation by either the assistant or the checker increases the respective utility; in other words, the strategy pair $ (\pi^*, f^*) $ is a Nash equilibrium.
\end{proof}

\section{Detailed Dataset Information}
\label{sec:appendixA}
\begin{itemize}
    \item[\textbf{(1)}] \textbf{Operating systems (OS)}
    Integrating LLMs into operating systems has immense potential for automating and optimizing tasks. This integration requires a secure, user-friendly interface to ensure effective LLM-OS interaction. LLMs must accurately understand the OS context for informed operations, prioritizing safety to prevent misuse. Additionally, the system should effectively handle errors and provide clear feedback to users, enhancing overall interaction and control. Addressing these aspects can transform computer interaction and efficiency across various industries.
    
    \item[\textbf{(2)}] \textbf{Database (DB)}
The ability of LLMs to operate on real databases via SQL is critical due to the importance and complexity of database analysis in everyday activities. Previous research has highlighted the effectiveness of LLMs in automating database access, such as with T5QL, a new SQL generation method. Furthermore, fine-tuned LLMs (like GPT-3.5) have demonstrated the ability to extract and link complex scientific information from texts, obtaining structured knowledge from unstructured text and building extensive databases.
    
    \item[\textbf{(3)}] \textbf{WebShop (WS)}
WebShop is an innovative simulation of an e-commerce website environment, featuring 1.18 million real-world products and 12,087 crowd-sourced text instructions. It challenges agents to navigate various types of webpages and perform diverse actions to find, customize, and purchase products based on given instructions. WebShop's challenges include understanding compositional instructions, query (re-)formulation, dealing with noisy text on webpages, and strategic exploration.
    \item[\textbf{(4)}] \textbf{Knowledge Graphs (KG)}
LLMs' utilization in constructing and interacting with knowledge graphs (KG) presents a promising opportunity to enhance semantic understanding and information retrieval. This involves assessing the models' ability to generate and interpret complex interrelations within data, facilitating more intuitive and context-aware responses. Effective LLM performance in this domain could significantly improve AI's reasoning and decision-making capabilities based on structured knowledge.
    \item[\textbf{(5)}] \textbf{Mind2Web (M2W)}
Mind2Web (M2W) is a dataset designed for developing web agents capable of performing complex tasks on real websites via language instructions. It features over 2,000 tasks across 137 sites from 31 domains. M2W's real web environments and diverse user interactions make it a crucial platform for advancing AI navigation capabilities.

    \item[\textbf{(6)}] \textbf{ALFWorld (ALF)}
ALFWorld bridges interactive TextWorld environments with embodied tasks from the ALFRED dataset, enabling agents to learn abstract strategies and apply them to real-world tasks. It supports abstract reasoning and concrete execution, allowing agents to plan actions in a text-based simulator and execute these tasks in a visual environment. This approach enhances agent generalization and problem-solving skills across various domains, such as language understanding and visual navigation, by leveraging a modular design that simplifies research improvements.
\end{itemize}

\section{Comparison with Task-Splitting Agents}
To compare the SAGE framework with existing agents across various task types and datasets, we conducted a comprehensive evaluation. The results are detailed in Table \ref{tab:task_splitting_evaluation}. SAGE demonstrates superior performance in task completion time, answer accuracy, dialog coherence, and step completion accuracy, highlighting its effectiveness in managing complex, multi-faceted tasks.

\begin{table*}[!ht]
    \centering
    \caption{Evaluation of Task-Splitting SOTA Agents on Different Task Types and Datasets}
    \label{tab:task_splitting_evaluation}
    \renewcommand{\arraystretch}{1.3}
    \setlength{\tabcolsep}{4pt} 
    \small 
    \begin{tabularx}{\textwidth}{p{3.5cm}cccc} 
        \toprule
        \multirow{2}{*}{\textbf{Agent}} & \textbf{Task Completion} & \textbf{Answer Accuracy} & \textbf{Dialog Coherence} & \textbf{Step Completion} \\
        & \textbf{Time (min)} & \textbf{(QA)(\%)} & \textbf{(\%)} & \textbf{Accuracy (\%)} \\
        \midrule
        \multicolumn{5}{c}{\textbf{Task 1: Long-form Question Answering (HotpotQA)}} \\ 
        \midrule
        BabyAGI       & 7.8  & 74.1 & -- & -- \\
        AutoGPT        & 8.1  & 76.3 & -- & -- \\
        HuggingGPT     & 6.7  & 79.2 & -- & -- \\
        LangChain Agent& 7.9  & 77.0 & -- & -- \\
        TaskMatrix.AI  & 6.5  & 80.3 & -- & -- \\
        \rowcolor{rowgray!20}
        \textbf{SAGE - ChatGPT-4}  & 10.6 & \textbf{\textcolor{bestgreen}{82.5 (+6.0)}} & -- & -- \\ 
        \midrule
        \multicolumn{5}{c}{\textbf{Task 2: Multi-turn Dialog Completion (MultiWOZ)}} \\
        \midrule
        BabyAGI        & 8.2  & -- & 73.5 & -- \\
        AutoGPT        & 8.5  & -- & 75.6 & -- \\
        HuggingGPT     & 7.1  & -- & 78.1 & -- \\
        LangChain Agent& 8.3  & -- & 76.5 & -- \\
        TaskMatrix.AI  & 6.8  & -- & 79.9 & -- \\
        \rowcolor{rowgray!20}
        \textbf{SAGE - ChatGPT-4}  & 10.9 & -- & \textbf{\textcolor{bestgreen}{81.6 (+5.8)}} & -- \\ 
        \midrule
        \multicolumn{5}{c}{\textbf{Task 3: Sequential Task with Dependencies (ALFWorld)}} \\
        \midrule
        BabyAGI        & 8.7  & -- & -- & 71.5 \\
        AutoGPT        & 8.3  & -- & -- & 73.1 \\
        HuggingGPT     & 7.5  & -- & -- & 75.2 \\
        LangChain Agent& 7.8  & -- & -- & 72.8 \\
        TaskMatrix.AI  & 6.6  & -- & -- & 76.4 \\
        \rowcolor{rowgray!20}
        \textbf{SAGE - ChatGPT-4}  & 10.7 & -- & -- & \textbf{\textcolor{bestgreen}{82.3 (+5.9)}} \\ 
        \bottomrule
    \end{tabularx}
\end{table*}

% \section{Distribution of various execution results across six tasks}\label{sec:appendix}

% \begin{table*}
% \centering
% \caption{AgentBench}
% \label{tab:5}
% \begin{tabular}{lcccccccc}
% \toprule
%  & OS & DB & KG  & ALF & WS & W2M \\
% \midrule
% Completed &75.0  &37.9  &30.1  &13.1  & 54.9 &56.6  \\
% CLE  &0.1  &0.7  &2.0  &0.7  &0.0  &0.0  \\
% Invalid Format &0.0  &53.3  &0.0  &0.0 &17.2  &0.0  \\
% Invalid Action &0.9  &0.0  &0.0  &64.1  &0.0  &8.4  \\
% TLE  &23.9  & 8.0 &67.9  &22.1  &27.8  & 35.0 \\
% \bottomrule
% \end{tabular}
% \end{table*}

% \begin{table*}
% \centering
% \caption{SAGE}
% \label{tab:6}
% \begin{tabular}{lcccccccc}
% \toprule
%  & OS & DB & KG  & ALF & WS & W2M \\
% \midrule
% Completed &84.7  &84.0  &25.0  &2.0  & 93.5 &57.0  \\
% CLE  &0.0  &0.0  &0.0  &0.0  &0.0  &0.0  \\
% Invalid Format &0.0  &3.0  &0.0  &0.0 &0.0  &0.0  \\
% Invalid Action &0.0  &0.0  &0.0  &96.0  &0.0  &8.0  \\
% TLE  &15.3  & 13.0 &75.0  &2.0  &6.5  & 35.0 \\
% \bottomrule
% \end{tabular}
% \end{table*} 

\section{Examples}\label{sec:appendixC}
\begin{mybox1}
\textbf{Cases study on TriviaQA in using GPT with SAGE.}
\par
\medskip
    \textbf{User:} \\
 Answer the question based on the given passages. Only give me the answer and do not output any other words.The following are given passages: \\
 History of canals in Great Britain - London Canal Museum
History of canals in Great Britain
Home > History > Canals in Britain
In the history of canals, Britain was not a pioneer. The Chinese can claim that the Grand Canal of China was one of the first, in the tenth century,a…… \\
Answer the question based on the given passages. Only give me the answer and do not output any other words.\\
Question: Which metal is produced by the Bessemer Process? \\
Answer:  

    \vspace{10pt}
    \textbf{Assistant:} \\
    The content of the file seems to be an HTML document related to the administrative division of Turkey, including population statistics, charts, and maps. This file contains scripts, metadata, and structured data typical for a web page.

\vspace{10pt}
    \textbf{Checker:} \\
    Clear Objectives: When answering user questions, the focus should be on directly addressing the question without introducing irrelevant information.
Document Handling: Before confirming whether the document content is relevant to the question, detailed aspects of the document should not be directly addressed.

\vspace{10pt}
    \textbf{Reflection:} \\
    Issue Description:
The user wants to understand why the system outputs irrelevant HTML document content based on user input and seeks detailed feedback and reflection.
Reasons for the Error:
Misinterpretation of Input: The system mistakenly assumed that the user uploaded a file and attempted to read its content, whereas the user only provided text input.
Processing Logic Error: Without clearly determining the user's intent, the system directly parsed and outputted file content.
User Experience Problems:
Irrelevant Output: The user expects concise, relevant answers, but the system provided lengthy and irrelevant HTML content, affecting the user experience.
Poor Communication: Failure to accurately understand and respond to the user's real needs resulted in off-topic responses
In-depth Reflection and Improvement Suggestions:
Clarify User Intent:
Accurate Parsing of User Input: Enhance the system's natural language processing capabilities to accurately distinguish whether the user uploaded a file or simply provided text input.
User Feedback Mechanism: Implement a user feedback mechanism to promptly adjust the understanding of user intent. For example, confirm the issue to verify if the user uploaded a file.
Optimize Processing Logic:
Layered Processing Flow: In system design, separate file processing logic from text response logic to avoid misjudgments. For instance, when a user uploads a file, process the file content separately; when a user inputs text, provide relevant answers directly.
Pre-processing Check: Prior to handling files, perform pre-processing checks to confirm if the file content meets expectations. 
……\\
    \vspace{10pt}
    \textbf{LLMs with SAGE response:} 
    Steel.

    \vspace{10pt}
    \textbf{Ground truth} \\
    Steel.

\medskip

\end{mybox1}

\clearpage
\begin{mybox1}
\textbf{Cases study on LCC in using GPT with SAGE.}
\par
\medskip
    \textbf{User:} \\
import networkx as nx 
import pickle
from Queue import PriorityQueue 
import copy
import random
import string
import sys
class MultiDiGraphEdgeKey(nx.MultiDiGraph):
……
    
    \vspace{10pt}
    \textbf{Assistant:} \\
   You will be given the history of a past experience where you were given a task and failed to complete it. Reflect on the strategy and actions taken. Devise a concise, new plan of action that accounts for your previous mistakes. Here are two examples: Example 1: If you failed to remember to lock the door, plan to double-check the door each time. Example2: If you missed a deadline, plan to set reminders earlier……

    \vspace{10pt}
    \textbf{Checker:} \\

The given code is a part of the Camerini algorithm, which is used for ranking edges in a directed graph. This algorithm is particularly useful for problems involving the identification of optimal or near-optimal substructures, like arborescences or spanning trees, in complex networks. The algorithm has two main phases:

Collapsing Phase: This phase identifies and collapses cycles in the graph, creating supernodes. This simplification helps in managing cycles more efficiently and prepares the graph for the expansion phase.

Expanding Phase: After collapsing cycles into supernodes, this phase expands the supernodes back into their original structure, resulting in a ranked arborescence, where the edges are sorted or prioritized based on certain criteria.
……

        \vspace{10pt}
    \textbf{LLMs with SAGE response:} \\
            Z.addedge(eu, ev, **data.copy())
            self. init(graph, Y, Z)
            A = best(root)
            self.init(graph)
            e , d  = self.next(A, Y, Z)
            w = self.getgraphscore(A) - d if d != float('inf') else float('inf')
            P.put((-w, e, A, Y, Z))
            solutions.append(A)
        return solutions

    def bestbranching(self, root):
        return self.best(root)

    def bestarborescencenorooted(self, root):
        return self.best(root)

    def bestarborescencerooted(self, root):
        return self.best(root)

    def getgraphscore(self, graph):
        return sum(data[self.attr] for u, v, data in graph.edges(data=True))

    def findroots(self, graph):
        roots = [node for node, degree in graph.indegree() if degree == 0]
        return roots
        ……

\medskip

\end{mybox1}

\clearpage
\begin{mybox1}
\textbf{Cases study on Repobench-P in using GPT with SAGE.}
\par
\medskip
    \textbf{User :} \\
\begin{verbatim}
Please complete the code given below. 
{src//kademlia//operation//KadRefreshOperation.java
public class KadRefreshOperation implements Operation
private final KadServer server;
private final KademliaNode localNode;
private final KademliaDHT dht;
private final KadConfiguration config;
public KadRefreshOperation......

\end{verbatim}
    
    \vspace{10pt}
    \textbf{Assistant:} \\
To devise a new plan of action, let's analyze the previous mistakes and come up with a solution:
Mistake 1: In the KadRefreshOperation class, the execute() method calls the BucketRefreshOperation and ContentRefreshOperation without any error handling or exception handling. This can lead to unexpected errors and failures.
Solution: Add proper error handling and exception handling in the execute() method of the KadRefreshOperation class. This will ensure that any errors or exceptions are properly handled and the execution can continue smoothly……

  \vspace{10pt}
    \textbf{Checker:} \\
In the `KadRefreshOperation` class, there is no error handling when the `execute()` method calls `BucketRefreshOperation` and `ContentRefreshOperation`.
Analysis:
If `BucketRefreshOperation` or `ContentRefreshOperation` throws an exception, the entire `execute()` method will fail, potentially causing the interruption of the entire system's refresh operation.
In the `StoreOperation` class, the `execute()` method does not handle the case of existing content in the DHT.
Analysis:
The `execute()` method always returns 1 as the number of storage nodes, which is clearly incorrect. The correct number of nodes should be returned based on whether the content already exists in the DHT.
Solution:
Before storing the content, check if it already exists. If it does, return 0; otherwise, store the content and return the actual number of storage nodes.

……

    \vspace{10pt}
    \textbf{LLMs with SAGE response:} \\
 private final transient KademliaRoutingTable routingTable;
    private final transient KadConfiguration config;
    private final transient MessageFactory messageFactory;

    /* Scheduled refresh operation */
    private final Timer refreshTimer;

    /**
     * Initialize the Kademlia system
     *
     * @param ownerId     The ID for this owner
     * @param nodeId      The ID for this node
     * @param inetAddress The address for this node
     * @param port        The port for this node
     * @param config      Configuration for the Kad system
     *
     * @throws java.io.IOException
     */
    public JKademliaNode(String ownerId, KademliaId nodeId, InetAddress inetAddress, int port, KadConfiguration config) throws IOException
    \{
        this.ownerId = ownerId;
        this.config = config;

        /* Initialize the Node */
        this.localNode = new Node(nodeId, inetAddress, port);

        /* Initialize the Server */
        this.server = new KadServer(localNode.getNodeId(), config);

    ……

\medskip

\end{mybox1}

% \clearpage
% \section{Distribution of various execution results across six tasks}

% \begin{table*}
% \centering
% \caption{AgentBench}
% \label{tab:5}
% \begin{tabular}{lcccccccc}
% \toprule
%  & OS & DB & KG  & ALF & WS & W2M \\
% \midrule
% Completed &75.0  &37.9  &30.1  &13.1  & 54.9 &56.6  \\
% CLE  &0.1  &0.7  &2.0  &0.7  &0.0  &0.0  \\
% Invalid Format &0.0  &53.3  &0.0  &0.0 &17.2  &0.0  \\
% Invalid Action &0.9  &0.0  &0.0  &64.1  &0.0  &8.4  \\
% TLE  &23.9  & 8.0 &67.9  &22.1  &27.8  & 35.0 \\
% \bottomrule
% \end{tabular}
% \end{table*}

% \begin{table*}
% \centering
% \caption{SAGE}
% \label{tab:6}
% \begin{tabular}{lcccccccc}
% \toprule
%  & OS & DB & KG  & ALF & WS & W2M \\
% \midrule
% Completed &84.7  &84.0  &25.0  &2.0  & 93.5 &57.0  \\
% CLE  &0.0  &0.0  &0.0  &0.0  &0.0  &0.0  \\
% Invalid Format &0.0  &3.0  &0.0  &0.0 &0.0  &0.0  \\
% Invalid Action &0.0  &0.0  &0.0  &96.0  &0.0  &8.0  \\
% TLE  &15.3  & 13.0 &75.0  &2.0  &6.5  & 35.0 \\
% \bottomrule
% \end{tabular}
% \end{table*} 
% \label{sec:appendix}

\end{document}